\pdfoutput=1
\documentclass[letterpaper]{article} % DO NOT CHANGE THIS
\usepackage{aaai23}  % DO NOT CHANGE THIS
\usepackage{times}  % DO NOT CHANGE THIS
\usepackage{helvet}  % DO NOT CHANGE THIS
\usepackage{courier}  % DO NOT CHANGE THIS
\usepackage[hyphens]{url}  % DO NOT CHANGE THIS
\usepackage{graphicx} % DO NOT CHANGE THIS
\usepackage{natbib}  % DO NOT CHANGE THIS AND DO NOT ADD ANY OPTIONS TO IT
\usepackage{caption} % DO NOT CHANGE THIS AND DO NOT ADD ANY OPTIONS TO IT
\usepackage{cuted}

\usepackage[utf8]{inputenc} % allow utf-8 input
\usepackage[T1]{fontenc}    % use 8-bit T1 fonts
\usepackage{xcolor}         % colors
\usepackage{hyperref}       % hyperlinks
\usepackage{url}            % simple URL typesetting
\usepackage{booktabs}       % professional-quality tables
\usepackage{amsfonts}       % blackboard math symbols
\usepackage{nicefrac}       % compact symbols for 1/2, etc.
\usepackage{microtype}      % microtypography
\usepackage[linesnumbered]{algorithm2e}
\usepackage{amsmath}
\usepackage{amssymb}
\usepackage{amsthm}
\newtheorem{theorem}{Theorem}
\newtheorem{lemma}{Lemma}
\newtheorem{fact}{Fact}
\newtheorem{definition}{Definition}
\newtheorem{corollary}{Corollary}

\usepackage{tikz}
\usetikzlibrary{arrows.meta, graphs}
\tikzset{
  pico/.style = {
    every node/.style = {
      draw,
      circle,
      semithick,
      inner sep = 0pt,
      minimum width = 0.7ex,
      fill = white
    },
    semithick
  },
    edge/.style = {
    semithick
  },
  arc/.style = {
    ->,
    semithick,
    >={[round,sep]Stealth}
  }
}

\makeatletter
\newcommand*{\indep}{%
  \mathbin{%
    \mathpalette{\@indep}{}%
  }%
}
\newcommand*{\nindep}{%
  \mathbin{%                   % The final symbol is a binary math operator
    \mathpalette{\@indep}{/}%
  }%
}
\newcommand*{\@indep}[2]{%
  \sbox0{$#1\perp\m@th$}%        box 0 contains \perp symbol
  \sbox2{$#1=$}%                 box 2 for the height of =
  \sbox4{$#1\vcenter{}$}%        box 4 for the height of the math axis
  \rlap{\copy0}%                 first \perp
  \dimen@=\dimexpr\ht2-\ht4-.2pt\relax
  \kern\dimen@
  \ifx\\#2\\%
  \else
    \hbox to \wd2{\hss$#1#2\m@th$\hss}%
    \kern-\wd2 %
  \fi
  \kern\dimen@
  \copy0 %                       second \perp
}
\makeatother

\newcommand{\Pa}{\textit{Pa}} 
 
\newcommand{\Ne}{\textit{Ne}}

\definecolor{ba.yellow}{RGB}{252,190,18}
\definecolor{ba.gray}{RGB}{153,153,156}
\definecolor{ba.blue}{RGB}{6,123,164}
\definecolor{ba.red}{RGB}{213,96,98}
\definecolor{ba.orange}{RGB}{233,116,81}
\definecolor{ba.pine}{RGB}{67,154,134}
\definecolor{ba.green}{RGB}{0, 168, 107}
\definecolor{ba.lightgreen}{RGB}{196,247,161}
\definecolor{ba.violet}{RGB}{88, 53, 94}

\newcommand{\plot}[4][0cm]{
  \draw[semithick, color=#3] (0,0)
  \foreach [count=\x] \y in {#2}{
    -- (\x,\y) node[circle, inner sep = 0pt, minimum width=1mm, draw=#3, fill=#3!50] {}
  } node[right=0.25cm, yshift=#1] {\small#4};
}

\title{Efficient Enumeration of Markov Equivalent DAGs\thanks{Extended
  version of paper accepted to the Proceedings of the 37th AAAI
Conference on Artificial Intelligence (AAAI-23).}}
\author{
  Marcel Wienöbst,\textsuperscript{\rm 1}
  Malte Luttermann,\textsuperscript{\rm 2}
  Max Bannach,\textsuperscript{\rm 1}
  Maciej Li\'{s}kiewicz\textsuperscript{\rm 1}
}
\affiliations{%University of L\"{u}beck, Germany\\
 \textsuperscript{\rm 1} Institute for Theoretical Computer Science, University of L\"{u}beck, Germany
 \\
 \textsuperscript{\rm 2} Institute of Information Systems, University of
 L\"{u}beck, Germany \\
  \{wienoebst@tcs, luttermann@ifis,
  bannach@tcs, liskiewi@tcs\}.uni-luebeck.de
}

\begin{document}
\maketitle\vspace*{-7.88mm}

\begin{abstract}
Enumerating the directed acyclic graphs (DAGs) of a Markov equivalence
class (MEC) is an important primitive in causal analysis. The central
resource from the perspective of computational complexity is the
delay, that is, the time an algorithm that lists all members of the
class requires between two consecutive outputs. Commonly used
algorithms for this task utilize the rules proposed by Meek (1995) or
the transformational characterization by Chickering (1995), both
resulting in %strictly 
superlinear delay. In this paper, we present the first
linear-time delay algorithm. On the theoretical side, we show that our
algorithm can be generalized to enumerate DAGs represented by models
that incorporate background knowledge, such as MPDAGs; on the practical side, we provide an efficient
implementation and evaluate it in a series of
experiments. Complementary to the linear-time delay algorithm, we also
%propose a new method with polynomial delay that 
provide intriguing
insights into Markov equivalence itself: All members of an MEC can be enumerated such that two successive DAGs have
structural Hamming distance at most three.
\end{abstract}

\section{Introduction}
Graphical causal models endow researchers with an intuitive and
mathematically sound language to infer causal relations between random
variables from observational and interventional data.  Directed
acyclic graphs (DAGs), whose edges encode direct causal influences between
the variables, belong to the most popular models and are used in many
areas of empirical
research~\citep{spirtes2000causation,rothman2008modern,pearl2009causality,koller2009probabilistic,Elwert2013}.
However, there is usually not a unique DAG that can be learned from
observational or limited experimental data as multiple models can encode the
same statistical properties. These DAGs form a \emph{Markov
  equivalence class} (MEC) and each of them explains the data equally
well~\citep{andersson1997characterization,pearl2009causality}.

Exploring the structural and quantitative properties of MECs
are challenging tasks in graphical causal analysis and, despite
extensive research efforts, several basic issues involving MECs fundamental to causal discovery remain
open as e.\,g., calculating the number of MECs on $n$ variables~\citep{gillispie2001enumerating,steinsky2003enumeration} or
enumerating them efficiently~\citep{chen2016enumerating}.

\begin{figure}
  \begin{tikzpicture}[scale=0.7]
    \node[inner sep = 2] (a) at (-0.5,-2) {$a$};
    \node[inner sep = 2] (b) at (1.5,-2) {$b$};
    \node[inner sep = 2] (c) at (-0.5,-3) {$c$};
    \node[inner sep = 2] (d) at (.5,-3) {$d$};
    \node[inner sep = 2] (e) at (1.5,-3) {$e$};
    \node[inner sep = 2] (f) at (-.5,-4) {$f$};
    \node[inner sep = 2] (g) at (1.5,-4) {$g$};
    %\node (l) at (.5,-1) {CPDAG};
    %\node (l) at (7,-2.5) {MEC};
    
    \graph[use existing nodes, edges = {edge}] {
      a -- c -- f;
      b -- e;
    };
    
    \graph[use existing nodes, edges = {arc}] {
      b -- d -- a;
      g -- d -- c;
      d -- f;
    };

    \draw (2.25,-5) -- (2.25,0);
    
    \node[inner sep = 2] (a) at (3,0) {$a$};
    \node[inner sep = 2] (b) at (5,0) {$b$};
    \node[inner sep = 2] (c) at (3,-1) {$c$};
    \node[inner sep = 2] (d) at (4,-1) {$d$};
    \node[inner sep = 2] (e) at (5,-1) {$e$};
    \node[inner sep = 2] (f) at (3,-2) {$f$};
    \node[inner sep = 2] (g) at (5,-2) {$g$};

    \graph[use existing nodes, edges = {arc}] {
      b -- d -- a;
      g -- d -- c;
      d -- f;
      a -- c -- f;
      b -- e;
    };
    \node[inner sep = 2] (a) at (6,0) {$a$};
    \node[inner sep = 2] (b) at (8,0) {$b$};
    \node[inner sep = 2] (c) at (6,-1) {$c$};
    \node[inner sep = 2] (d) at (7,-1) {$d$};
    \node[inner sep = 2] (e) at (8,-1) {$e$};
    \node[inner sep = 2] (f) at (6,-2) {$f$};
    \node[inner sep = 2] (g) at (8,-2) {$g$};

    \graph[use existing nodes, edges = {arc}] {
      b -- d -- a;
      g -- d -- c;
      d -- f;
      c -- f;
      c -- a;
      b -- e;
    };
    \node[inner sep = 2] (a) at (9,0) {$a$};
    \node[inner sep = 2] (b) at (11,0) {$b$};
    \node[inner sep = 2] (c) at (9,-1) {$c$};
    \node[inner sep = 2] (d) at (10,-1) {$d$};
    \node[inner sep = 2] (e) at (11,-1) {$e$};
    \node[inner sep = 2] (f) at (9,-2) {$f$};
    \node[inner sep = 2] (g) at (11,-2) {$g$};

    \graph[use existing nodes, edges = {arc}] {
      b -- d -- a;
      g -- d -- c;
      d -- f;
      f -- c -- a;
      b -- e;
    };
    \node[inner sep = 2] (a) at (3,-3) {$a$};
    \node[inner sep = 2] (b) at (5,-3) {$b$};
    \node[inner sep = 2] (c) at (3,-4) {$c$};
    \node[inner sep = 2] (d) at (4,-4) {$d$};
    \node[inner sep = 2] (e) at (5,-4) {$e$};
    \node[inner sep = 2] (f) at (3,-5) {$f$};
    \node[inner sep = 2] (g) at (5,-5) {$g$};

    \graph[use existing nodes, edges = {arc}] {
      b -- d -- a;
      g -- d -- c;
      d -- f;
      a -- c -- f;
      e -- b;
    };
    
    \node[inner sep = 2] (a) at (6,-3) {$a$};
    \node[inner sep = 2] (b) at (8,-3) {$b$};
    \node[inner sep = 2] (c) at (6,-4) {$c$};
    \node[inner sep = 2] (d) at (7,-4) {$d$};
    \node[inner sep = 2] (e) at (8,-4) {$e$};
    \node[inner sep = 2] (f) at (6,-5) {$f$};
    \node[inner sep = 2] (g) at (8,-5) {$g$};

    \graph[use existing nodes, edges = {arc}] {
      b -- d -- a;
      g -- d -- c;
      d -- f;
      c -- f;
      c -- a;
      e -- b;
    };

    \node[inner sep = 2] (a) at (9,-3) {$a$};
    \node[inner sep = 2] (b) at (11,-3) {$b$};
    \node[inner sep = 2] (c) at (9,-4) {$c$};
    \node[inner sep = 2] (d) at (10,-4) {$d$};
    \node[inner sep = 2] (e) at (11,-4) {$e$};
    \node[inner sep = 2] (f) at (9,-5) {$f$};
    \node[inner sep = 2] (g) at (11,-5) {$g$};

    \graph[use existing nodes, edges = {arc}] {
      b -- d -- a;
      g -- d -- c;
      d -- f;
      f -- c -- a;
      e -- b;
    };
  \end{tikzpicture}
  \caption{A Markov equivalence class (MEC) on the right, which
    consists of six DAGs. This class is represented by the left CPDAG,
    which uniquely represents the MEC by including
    undirected edges if two DAGs differ in their direction.}
  \label{fig:mecexample}
\end{figure}
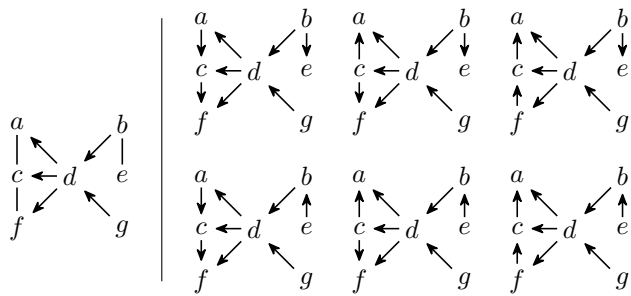

In this work, we study the properties of a
\emph{single} MEC, encoded as a \emph{completed partially directed acyclic graph} (CPDAG)
\cite{andersson1997characterization},
see Fig.~\ref{fig:mecexample} for an
example. The CPDAG
representation is often learned directly by causal discovery algorithms~\cite{spirtes2000causation,chickering2002optimal}
and, thus, it is of high practical value to offer efficient implementations
for their analysis. Recently, \citet{wienobst2021counting} have shown that computing 
the size of an MEC as well as uniformly sampling from it can be done in
polynomial time. We deal with the closely related problem of
\emph{enumerating the DAGs in an MEC} given its CPDAG, that is listing each
member of the MEC exactly once. This task is an
important primitive in causal analysis,
used as a subroutine to solve more complex problems in software
packages such as
\texttt{pcalg}~\cite{kalisch2012causal} and
\texttt{causaldag}~\cite{squires2018causaldag}. Enumeration of an MEC's members 
can be applied to solve many important downstream tasks in causal inference. %e.\,g, 
For example, one can estimate the causal effect of the exposure variable on the outcome 
for each DAG in the equivalence class, which is learned from the observed data~\citep{maathuis2009estimating}.
One could also check for every DAG whether it conforms to additional domain
information or background knowledge in order to find the most plausible
DAG~\cite{Meek1995}, or select intervention targets to
distinguish between certain DAGs in the class~\cite{he2008active,hauser2012characterization}. 
While there are custom algorithms, which avoid the use of explicit enumeration (sometimes
by settling for approximate solutions), for many
of these cases, it remains a flexible and very general tool that can be
utilized even when other methods fail. The main drawback is, of course, its high computational cost, which we aim to address in this work.

Any method for enumerating the DAGs in an MEC
requires exponential time in the worst-case, due to the basic
fact that there are classes with exponential size.
%\footnote{In this manuscript, we refer
%with $n$ to the number of variables and with $m$ to the number of edges.}
%$O(n!)$.
A crucial feature from a computational perspective is the
\emph{delay:}
the algorithm's run-time between two consecutive output DAGs.
Another desirable property would be that the  subsequent DAGs smoothly 
change their structure, i.\,e., share most of their edge orientations,
which constitutes a more plausible enumeration from the causal point of view.
In the present work, we take both these aspects into account.

To the best of our knowledge, no study has been published that
performs a systematic analysis of the enumeration
problem, including its algorithmic aspects. One commonly used 
folklore algorithm utilizes the rules
proposed by~\citet{Meek1995}
to transform a causal graph (e.\,g., a CDPAG or PDAG) into its maximal 
orientation.
Applying these rules has the property that any remaining
undirected edge $a - b$ is oriented $a \rightarrow b$ in at least one and $a
\leftarrow b$ in another DAG represented by the graph. Consequently,
the DAGs can be enumerated by successively trying both possible
orientations. This yields a
polynomial delay algorithm, but the degree of the corresponding polynomial is rather
large since the Meek rules have to be applied at
\emph{every} step.
Another folklore approach is based upon the transformational characterization
of MECs given by~\citet{chickering1995transformational},
which states that two DAGs in the same MEC can be transformed into each
other by successive single-edge reversals. Hence, the MEC can be
explored through such edge reversals starting from an arbitrary DAG in
the class. The issue with this approach is
that already output DAGs need to be stored and every time a
new DAG is explored it has to be checked that it has not been
output before. This leads to a relatively large delay and memory demand. As both algorithms (which we call
\textsc{meek-enum} and \textsc{chickering-enum}) have not been
explicitly stated in a publication, we give a formal
description of both in Appendix~\ref{appendix:related} as
well as a rigorous analysis of their delay.

The main contribution of this paper is the first $O(n+m)$-delay algorithm that, for a given CPDAG representing an MEC, lists all
members of the class.\footnote{We denote the number of vertices by $n$ and the
  number of edges by $m$.} 
We also show that the algorithm can be generalized to enumerate DAGs 
represented by a PDAG or MPDAG  -- causal models incorporating background knowledge.
To achieve these results, we utilize the
Maximum Cardinality Search (MCS)~\cite{tarjan1984simple} originating from the chordal graph literature. 
In addition to the theoretical results, we give an efficient practical
implementation, which is significantly faster than implementations of
\textsc{meek-enum} and \textsc{chickering-enum}.

We also propose a complementary method
with the property that during enumeration subsequent DAGs gradually change their
structure. This method utilizes the results
by~\cite{chickering1995transformational}, but
performs the traversal of the MEC in a more refined way.
Using such an approach, it is possible to output all Markov equivalent 
DAGs in sequence with the property that two successive DAGs have \emph{structural Hamming
distance} (SHD) at most three. This result is tight in the
sense that there are MECs whose members cannot be
enumerated in a sequence with maximal distance at most two.
We  also show that our ideas can be used in the more
general setting of enumerating maximal ancestral graphs (MAGs) which 
encode conditional independence relations in DAG models with latent variables \cite{richardson2002ancestral}.

\section{Preliminaries}
%\paragraph{General Backgrounds.}
A graph $G = (V, E)$ consists of a set of vertices $V$ and a set of
edges $E \subseteq V \times V$. An edge $u-v$ is undirected if
$(u, v), (v, u) \in E$ and directed $u \rightarrow v$ if $(u,v) \in E$
and $(v,u) \not\in E$.  Vertices linked by an edge (of any type) are
\emph{adjacent} and \emph{neighbors} of each other. We say that $u$ is
a \emph{parent} of $v$ if $u\rightarrow v$.  We denote by $\Pa(v)$
and $\Ne(v)$ the set of parents and neighbors of~$v$. The \emph{degree} $\delta(v)$ of
vertex $v$ is the number of its neighbors $|\Ne(v)|$.
The structural
Hamming distance (SHD) of $G_1 = (V,E_1)$ and
$G_2 = (V,E_2)$ is denoted by $\mathrm{shd}(G_1, G_2)$ and defined as
number of pairs $(a,b) \in V^2$ with differing edge relations, i.\,e.,
$E_1 \cap \{(a,b), (b,a)\} \neq E_2 \cap \{(a,b), (b,a)\}$.  Given a
graph $G = (V,E)$ and a vertex set $S$, the \emph{induced
  subgraph}~$G[S]$ contains the edges $E \cap (S \times S)$ from $G$
that are incident only to vertices in $S$. The union of a set of graphs
$\{G_1 = (V, E_1), \dots, G_k = (V, E_k)\}$ is the graph
$G = (V, \bigcup_{i=1}^k E_k)$.  A path $\pi$ between two vertices~$v_1$ and~$v_p$ is a sequence of distinct vertices
$\pi=\langle v_1,\ldots,v_p\rangle$ with $p\ge 2$ such that each
vertex $v_i$ is adjacent to $v_{i+1}$ for $i=1,\ldots,p-1$.  An \emph{undirected
connected component} is a maximal induced subgraph in which every pair of vertices is
connected by a path of undirected edges. 
%A subgraph containing only a single node is also
%a connected component.
A path of the form
$v_1\rightarrow v_2\rightarrow \ldots\rightarrow v_p$ is directed or
\emph{causal}. A graph is \emph{acyclic} if there is no directed path from a vertex
$u$ to $v$ with $v\rightarrow u$.  An acyclic graph with only directed
edges is called a DAG.  An undirected graph is called \emph{chordal}
if no subset of four or more vertices induces an undirected cycle.

The \emph{skeleton} of $G$, denoted by
$\mathrm{skel}(G)$, is a graph with the same
vertex set in which every edge is replaced by an undirected edge. A
\emph{$v$-structure} is
an ordered triple of vertices $(u,c,v)$ that induces the subgraph
$u \rightarrow c \leftarrow v$.
A Markov equivalence class (MEC) consists of DAGs
encoding the same set of conditional independence relations among the variables. 
\citep{verma1990equivalence,frydenberg1990chain} showed that 
two DAGs are Markov equivalent iff
they have the same skeleton and the same v-structures.
An MEC can be represented by a CPDAG (\emph{completed partially
  directed acyclic graph}), % also known as essential graph), 
which is the union graph of the DAGs in the equivalence class it
represents. The set $[G]$ denotes all DAGs in the MEC represented by CPDAG $G$.

Subclasses of MECs can be represented by \emph{partially directed acyclic
  graphs} (PDAGs), which are restricted only in that they may not contain
a directed cycle, and \emph{maximally oriented PDAGs} (MPDAGs), which
consist of PDAGs closed under the four \emph{Meek rules}. Explicitly,
we only use the first Meek rule in this work, which states that an induced subgraph
$a \rightarrow b - c$ is oriented into $a \rightarrow b \rightarrow c$.
A formal description of all four Meek rules can be found in Appendix~\ref{appendix:meeksubsection}.

%\paragraph{Algorithmic Preliminaries.}
The input of the enumeration algorithms is the representation of an MEC
in form of its CPDAG~$G$. Hence, before approaching the enumeration task, it is
important, first and foremost, to understand how one can derive a DAG in $[G]$
from the CPDAG representation (this is also called the \emph{extension
  task} and such a DAG is called a \emph{consistent extension}
of $G$). To extend $G$ into a consistent DAG, it
is necessary to find an orientation of its undirected edges. This orientation needs to be acyclic and
contain the same v-structures as $G$. 

\begin{fact}[\citet{andersson1997characterization,He2015}] \label{fact:reduction}
  The undirected components of a CPDAGs are \emph{undirected and
    connected chordal graphs} (UCCGs). Acyclically orienting each UCCG
  independently, without introducing a v-structure, gives a DAG in $[G]$.
\end{fact}

The acyclic orientations without a v-structure of a chordal graph are
called \emph{AMOs} (acyclic moral orientations).
%In addition to the above fact, 
\emph{Every} DAG in $[G]$ may
be computed by finding appropriate AMOs for the UCCGs. In Fig.~\ref{fig:mecexample},
to obtain the DAG at the top left, the UCCG $a - c - f$ is oriented
as AMO $a \rightarrow c \rightarrow f$ and $b - e$ is
oriented as $b \rightarrow e$. For the former there are three AMOs,
for the latter two, corresponding to the six DAGs in the MEC (as the
UCCGs can be oriented independently). Thus, when tackling the
enumeration task for CPDAGs, it suffices to enumerate
the AMOs of a chordal graph.

\begin{fact}[Implicit in \citet{wienobst2021counting}] \label{fact:counting}
  Every AMO of a UCCG $G$ can be obtained by orienting the edges
  according to an MCS ordering ($a \rightarrow b$ if $a$ comes before $b$ in the ordering).
\end{fact}
An MCS ordering is a linear ordering of the vertices produced by
running the graph traversal algorithm \emph{Maximum Cardinality
  Search} (MCS)~\cite{tarjan1984simple}, which was originally proposed for testing
chordality of a graph in linear time.
Notably, the reverse direction holds as well, that is,
an MCS ordering will always produce an AMO of $G$.
A brief introduction to chordal graphs, AMOs and the MCS algorithm is
given in Appendix~\ref{appendix:mcssubsection}.

The output of an MCS
depends on the choices of the ``next'' vertex to visit in each
step of the graph traversal. This vertex is taken from the set of 
vertices with the largest number of
already visited neighbors (called the vertices with \emph{highest
  label}) and, from Fact~\ref{fact:counting}, we can conclude
that there are such choices, which may produce any AMO of a given
UCCG.

\section{Enumerating AMOs with Linear Delay}
The observations from the previous section yield a new approach:
Instead of producing a single AMO by choosing an arbitrary vertex in
each step of the MCS, we perform multiple choices and recur for each
of them~--~eventually listing all AMOs. There is one pitfall however:
We \emph{cannot} simply choose \emph{every} vertex from the
highest-label set one after the other as some graphs would be output
multiple times. Fig.~\ref{fig:disconnect} illustrates this issue: If
vertex $a$ has been visited, the next vertex could be $b$, $c$, $d$,
$e$, or $f$ (all have one visited neighbor, namely $a$). But choosing,
say, $b$ or $e$ may lead to the same AMO as the order of $b$ and $e$
after choosing $a$ is irrelevant.

% MW:
% In order to save space, we could remove Fig. 2 and use Fig. 1 as
% simpler example. There,  for a - c - f, after choosing c, the order
% of a and f is irrelevant.
% This is however also less insightful

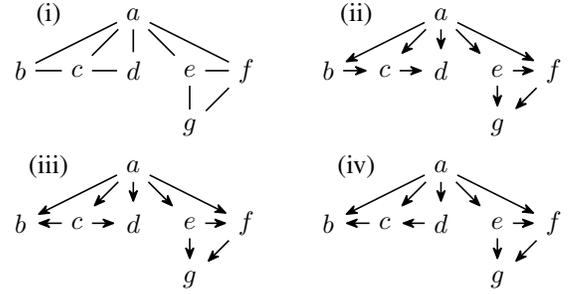
\begin{figure}
  \begin{center}
    \begin{tikzpicture}[xscale=0.75, yscale=0.75]
      \node (a) at (0,0) {$a$};
      \node (b) at (-2,-1) {$b$};
      \node (c) at (-1,-1) {$c$};
      \node (d) at (0,-1) {$d$};
      \node (e) at (1,-1) {$e$};
      \node (f) at (2,-1) {$f$};
      \node (g) at (1,-2) {$g$};
      \node (l) at (-1.5,0) {(i)};
      
      \graph[use existing nodes, edges = {edge}] {
        a -- b;
        a -- c;
        a -- d;
        a -- e;
        a -- f;
        b -- c;
        c -- d;
        e -- f;
        e -- g;
        f -- g;
      };
    \end{tikzpicture} \hspace*{0.5cm}
    \begin{tikzpicture}[xscale=0.75, yscale=0.75]
      \node (a) at (0,0) {$a$};
      \node (b) at (-2,-1) {$b$};
      \node (c) at (-1,-1) {$c$};
      \node (d) at (0,-1) {$d$};
      \node (e) at (1,-1) {$e$};
      \node (f) at (2,-1) {$f$};
      \node (g) at (1,-2) {$g$};
      \node (l) at (-1.5,0) {(ii)};
      
      \graph[use existing nodes, edges = {arc}] {
        a -- b;
        a -- c;
        a -- d;
        a -- e;
        a -- f;
        b -- c;
        c -- d;
        e -- f;
        e -- g;
        f -- g;
      };
    \end{tikzpicture} 
    \begin{tikzpicture}[xscale=0.75,yscale=0.75]
      \node (a) at (0,0) {$a$};
      \node (b) at (-2,-1) {$b$};
      \node (c) at (-1,-1) {$c$};
      \node (d) at (0,-1) {$d$};
      \node (e) at (1,-1) {$e$};
      \node (f) at (2,-1) {$f$};
      \node (g) at (1,-2) {$g$};
      \node (l) at (-1.5,0) {(iii)};
      
      \graph[use existing nodes, edges = {arc}] {
        a -- b;
        a -- c;
        a -- d;
        a -- e;
        a -- f;
        c -- b;
        c -- d;
        e -- f;
        e -- g;
        f -- g;
      };
    \end{tikzpicture} \hspace*{0.5cm}
    \begin{tikzpicture}[xscale=0.75,yscale=0.75]
      \node (a) at (0,0) {$a$};
      \node (b) at (-2,-1) {$b$};
      \node (c) at (-1,-1) {$c$};
      \node (d) at (0,-1) {$d$};
      \node (e) at (1,-1) {$e$};
      \node (f) at (2,-1) {$f$};
      \node (g) at (1,-2) {$g$};
      \node (l) at (-1.5,0) {(iv)};
      
      \graph[use existing nodes, edges = {arc}] {
        a -- b;
        a -- c;
        a -- d;
        a -- e;
        a -- f;
        c -- b;
        d -- c;
        e -- f;
        e -- g;
        f -- g;
      };
    \end{tikzpicture}
  \end{center}
  \caption{An example showing for the chordal graph (i) that, at a given step of the
  algorithm, not all vertices in the highest-label set can be
  chosen. If the MCS starts with
  vertex $a$, all neighbors $b, c, d, e, f$ have the same label,
  namely 1. While an MCS may choose any one of them, we \emph{cannot} choose \emph{all} one-after-the-other in our
  enumeration. Choosing $b$ or $e$ as the second vertex may yield the same
  AMO as $a, b, c, d, e, f, g$ and $a, e, f, g, b, c, d$ are both topological
  orderings of (ii). However, choosing highest-label
  vertices from one connected component in $G[\{b,c,d,e,f, g\}]$ such as $\{b,c,d\}$
  one-after-another will yield distinct AMOs (in (iii) and (iv) AMOs with
  $c$, $d$ chosen as second vertex are given).}
\label{fig:disconnect}
\end{figure}

This issue
can be addressed as follows: While the choice of the first vertex $v$
from the highest-label set is arbitrary, every other vertex
$x$ has to be connected to $v$ in the remaining graph
(the induced subgraph over the unvisited vertices). If they are
connected, then the order of $v$ and $x$ matters.
% (as there are no
% v-structures we have to orient either $v \rightarrow \dots \rightarrow x$ or
% $v \leftarrow \dots \leftarrow x$ depending on whether $v$ or $x$ is
% chosen first) and we have to choose $x$ before $v$ as well.
Otherwise, choosing $x$ instead of $v$ would lead to duplicate AMOs
being output. In our example, this means that if~$b$ is the first
considered vertex with highest label after $a$ has been visited, the
other choices we would consider are $c$ and~$d$ as these are the
vertices reachable from $b$ in the induced subgraph over the unvisited
vertices.

\begin{lemma}\label{lemma:amoconn}
  Given a chordal graph $G = (V, E)$ and the sequence of previously visited vertices
  $\tau$ with $|\tau| = k < n$ produced by an MCS with current
  highest-label set $S$.
  \begin{enumerate}
    \item\label{lemma:amoconn:a} If $x,y \in S$ are connected in $G[V
      \setminus \tau]$, the
      set of AMOs produced by choosing
      $x$ next is disjoint from the set produced by choosing $y$ next.
    \item\label{lemma:amoconn:b} If $x,y \in S$ are unconnected in
      $G[V \setminus \tau]$, any
      AMO produced by choosing $y$ as the next vertex can be
      produced by choosing a vertex in $S$ connected to $x$ in $G[V \setminus
      \tau]$ next.
  \end{enumerate}
\end{lemma}

\begin{proof}
  We show item~\ref{lemma:amoconn:a} first and let $p$ be the shortest path between $x$ and $y$ in
  $G[V \setminus \tau]$. Since AMOs do not contain v-structures, any AMO
  choosing $x$ before $y$ must orient $p$ as
  \(
  x\rightarrow \dots \rightarrow y
  \) while any AMO with $y$ next yields
  \(
  x \leftarrow \dots \leftarrow y
  \).

  For item~\ref{lemma:amoconn:b} let $C_1, \dots, C_k$ be the
  connected components of $G[V \setminus \tau]$ with $x \in C_1$. Any
  topological ordering with prefix~$\tau$ can be rewritten as
  $\tau, C_{\pi(1)}, \dots, C_{\pi(k)}$ for an arbitrary permutation
  $\pi$~--~in particular for $\pi = \mathrm{id}$.
\end{proof}

The following Lemma provides a simplified way of testing whether two
vertices of highest label are connected.
\begin{lemma}\label{lemma:subgraph:over:s}
  Given a connected chordal graph $G = (V, E)$ and a sequence of
  visited vertices $\tau$ produced by an MCS with the current
  highest-label set $S$. Vertices $x,y \in S$ are connected in $G[S]$
  iff they are connected in $G[V \setminus \tau]$.
\end{lemma}

\begin{algorithm}[t]
  \caption{Linear-time delay enumeration algorithm \textsc{mcs-enum}
    for listing all AMOs of a chordal graph.}
  \label{alg:mcs}
  \DontPrintSemicolon
  \SetKwInOut{Input}{input}\SetKwInOut{Output}{output}
  \SetKwFor{Rep}{repeat}{}{end}
  \Input{A UCCG $G = (V,E)$.}
  \Output{All AMOs of $G$.}
  \SetKwFunction{FEnum}{enumerate}
  \SetKwProg{Fn}{function}{}{end}
  \SetKwRepeat{Do}{do}{while}

  \BlankLine
  $A :=$ array of $n$ initially empty sets \; \label{line:startinit}
  $\tau :=$ empty list \;
  $A[0] := V$ \;
  $\FEnum(G, A, \tau)$ \; \label{line:endinit}
  \BlankLine
  
  \Fn{\FEnum{$G$, $A$, $\tau$}}{
    \If{$|\tau| = n$ \label{line:start}}{Output AMO of $G$ according
      to ordering
      $\tau$ \label{line:out}}
    $i := \text{highest index of non-empty set in } A$ \;
    $v := \text{any vertex from } A[i]$ \;
    $x := v$ \;

    \Do{$\text{$R$ is non-empty}, x := \text{pop}(R)$ \label{line:while}}{
      \label{line:do}
      $\text{delete } x \text{ from } A[i]$ \; \label{line:removefromAi}
      $\text{append } x \text{ to } \tau$ \; \label{line:appendtotau}
      \ForEach{$w \in (\Ne(x) \setminus \tau)$ \label{line:startfor1}}{
        $j := \text{index of set in } A \text{ that } w \text{ is
          in}$ \;
        $\text{delete } w \text{ from } A[j]$ \;
        $\text{insert } w \text{ in } A[j+1]$ \;
      } \label{line:endfor1}
      \FEnum{$G$, $A$, $\tau$}\; \label{line:rec}
      \ForEach{$w \in (\Ne(x) \setminus \tau)$ \label{line:startfor2}}{
        $j := \text{index of set in } A \text{ that } w \text{ is
          in}$ \;
        $\text{delete } w \text{ from } A[j]$ \;
        $\text{insert } w \text{ in } A[j-1]$ \;
      } \label{line:endfor2}
      $\text{insert } x \text{ in } A[i]$ \;
      $\text{pop } x \text{ from } \tau$ \; \label{line:endreversals}
      \If{$x = v$}{
        $R:= \{ a \mid a \text{ reachable from } v \text{ in } G[A[i]]
        \}$ \;         \label{line:reach}
      }
      
    }
  }
\end{algorithm}

Algorithm \textsc{mcs-enum} utilizes these results to enumerate the
AMOs of a chordal graph. Lines~\ref{line:startinit}
to~\ref{line:endinit} in Algorithm~\ref{alg:mcs} initialize the
necessary data structures for an MCS (in particular an array $A$,
which includes the set of vertices with $i$ visited neighbors at
index $A[i]$). Afterward, the recursive function \texttt{enumerate} is
called. It first chooses any vertex $v$ from the vertices with most
visited neighbors (just as a normal MCS). This vertex is then removed
from $A[i]$ in line~\ref{line:removefromAi}, it is appended to $\tau$
in line~\ref{line:appendtotau}, which stores the traversal sequence
(later used to impose edge directions based on its ordering) and the
neighbors are moved from $A[j]$ to $A[j+1]$ in
lines~\ref{line:startfor1} to~\ref{line:endfor1}. The recursive
calls to \texttt{enumerate} are repeated until, at some point, the first AMO
is output in line~\ref{line:out} (up to this point there is no
difference to an MCS).

After the recursive call, however, the changes are reversed in
lines~\ref{line:startfor2} to~~\ref{line:endreversals} and then, in
line~\ref{line:reach}, the vertices reachable from
$v$ are computed. These are then iterated in the do-while loop,
meaning we also recursively go through the AMOs produced by choosing
those vertices instead of $v$, as discussed above. Note that reachability in
line~\ref{line:reach} is only performed once in a call of
\texttt{enumerate} for the initial vertex $v$.

\begin{theorem}\label{theorem:mcsenumcorrectness}
  Given a chordal graph $G$, \textsc{mcs-enum} enumerates all
  AMOs of $G$.
\end{theorem}

\begin{proof}
  Every DAG output in line~\ref{line:out} is
  an AMO, as it is generated by a linear ordering
  produced by an MCS. This holds as any chosen vertex is from
  the highest-index non-empty set in~$A$.
  To see that the algorithm outputs \emph{all} AMOs of $G$, recall that every AMO can be
  represented by an MCS ordering by Fact~\ref{fact:counting}.
  In principle, Algorithm~\ref{alg:mcs}
  considers all possible courses an MCS could take, except the pruning
  of vertices unreachable from $v$. By
  Lemma~\ref{lemma:subgraph:over:s}, it suffices to inspect only
  connected vertices in $G[A[i]]$ and by item~\ref{lemma:amoconn:b} of
  Lemma~\ref{lemma:amoconn} those
  unreachable vertices would not lead to any new AMO.

  Finally, we argue that no AMO is output twice. Every output is
  obtained by constructing a directed graph based on the
  ordering given by the graph traversal.
  Assume for the sake of
  contradiction that we have two such sequences $\tau_1$ and $\tau_2$ representing the same
  AMO. Let $x$ and~$y$ be the vertices in $\tau_1$ and $\tau_2$ at the first
  differing position, respectively.
  Note that $x$ and $y$ are connected and, hence, by item~\ref{lemma:amoconn:a} of
  Lemma~\ref{lemma:amoconn} it follows that $\tau_1$ and $\tau_2$
  yield different AMOs.
\end{proof}

%Moreover, we show that \textsc{mcs-enum} has linear-time delay:

\begin{theorem} \label{theorem:wcdelay}
  \textsc{mcs-enum} has worst-case delay $O(n+m)$.
\end{theorem}

\begin{proof}
  Let us partition the steps between two outputs in three phases: (i)
  the recursion goes ``upwards'' from an
  output; (ii) it reaches its ``top'' in the recursion tree; and
  (iii) the recursion goes ``downwards'' towards the next output.

  We show that each phase runs in time $O(n+m)$.
  In phase~(i), lines~\ref{line:startfor2} to~\ref{line:while} are
  executed and the do-while loop stops (otherwise we would be in phase (ii)).
  The for-loop in
  lines~\ref{line:startfor2} to~\ref{line:endfor2} has time complexity
  $O(\delta(x))$. Moreover, the
  reachability query executed in line~\ref{line:reach} does not yield any
  vertices (otherwise the do-while loop would continue), meaning it takes
  time $O(\delta(x))$ (all neighbors of $v$ are checked once and the search
  stops). The run-time is therefore $O(m)$ as
  every edge is considered at most twice.

  The main costs of phase (iii) are produced by the for-loop in
  lines~\ref{line:startfor1} to~\ref{line:endfor1}, which requires
  time $O(\delta(x))$ leading to an overall time of $O(m)$.
  Both for-loops are executed in phase (ii), resulting in
  overall time $O(\delta(x))$. The reachability query from $v$ in
  line~\ref{line:reach} costs time $O(m)$. As this is done only once,
  we obtain a worst-case delay of $O(n+m)$.
\end{proof}

The result immediately generalizes to CPDAGs.

\begin{theorem}
  The Markov equivalence class $[G]$ of a CPDAG $G$
  can be enumerated with worst-case delay $O(n+m)$.
\end{theorem}

\begin{proof}
  Algorithm~\ref{alg:mcs} also works for unconnected chordal graphs
  without any modifications. Hence, it can be called for the
  graph obtained by removing all directed edges of $G$. After computing an AMO
  of this graph, the
  directed edges can be re-added and the output is a member of $[G]$ produced
  with delay $O(n+m)$. The correctness follows from Fact~\ref{fact:reduction}.
\end{proof}

\section{PDAGs and MPDAGs}
One can naturally generalize the enumeration task to only consider members
of the MEC, which conform to certain background knowledge given in the form of
additional directed edges. Such a subclass of an
MEC is commonly represented by an MPDAG (or PDAG).
As before, a DAG is a consistent extension of $G$ if it
has the same skeleton and v-structures.%~\cite{dor1992simple}.
It is easy to see that \textsc{meek-enum} works in this generalized
setting as well, as does \textsc{chickering-enum} (see also
Corollary~\ref{corollary:pdagchickering} in Section~\ref{section:anotherapproach}).
In this section, we show how
\textsc{mcs-enum} can be adapted, first to handle MPDAGs and, building
on this, PDAGs.

\begin{definition}
  We term the graph induced by
  the vertices of a undirected component in an MPDAG a \emph{bucket}.
\end{definition}

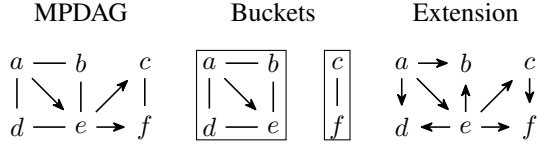
\begin{figure}
  \centering
  \begin{tikzpicture}[scale=0.85]
    \node (a) at (0,0) {$a$};
    \node (b) at (1,0) {$b$};
    \node (c) at (2,0) {$c$};
    \node (d) at (0,-1) {$d$};
    \node (e) at (1,-1) {$e$};
    \node (f) at (2,-1) {$f$};
    \node (l) at (1,0.8) {MPDAG};

    \graph[use existing nodes, edges = {edge}] {
      a -- b -- e -- d -- a;
      c -- f;
    };
    
    \graph[use existing nodes, edges = {arc}] {
      a -- e -- c;
      e -- f;
    };    
  \end{tikzpicture} \hspace*{0.25cm}
  \begin{tikzpicture}[scale=0.85]
    \node (a) at (0,0) {$a$};
    \node (b) at (1,0) {$b$};
    \node (c) at (2,0) {$c$};
    \node (d) at (0,-1) {$d$};
    \node (e) at (1,-1) {$e$};
    \node (f) at (2,-1) {$f$};
    \node (l) at (1,0.8) {Buckets};

    \graph[use existing nodes, edges = {edge}] {
      a -- b -- e -- d -- a;
      c -- f;
    };
    
    \graph[use existing nodes, edges = {arc}] {
      a -- e;
    };

    \draw (-0.2,0.2) rectangle (1.2,-1.2);
    \draw (1.8,0.2) rectangle (2.2,-1.2);
  \end{tikzpicture} \hspace*{0.25cm}
  \begin{tikzpicture}[scale=0.85]
    \node (a) at (0,0) {$a$};
    \node (b) at (1,0) {$b$};
    \node (c) at (2,0) {$c$};
    \node (d) at (0,-1) {$d$};
    \node (e) at (1,-1) {$e$};
    \node (f) at (2,-1) {$f$};
    \node (l) at (1,0.8) {Extension};
    
    \graph[use existing nodes, edges = {arc}] {
      a -- e -- b;
      a -- b;
      a -- d;
      e -- d;
      e -- c;
      e -- f;
      c -- f;
    };
  \end{tikzpicture}
  \caption{The MPDAG on the left, contains two buckets with
    chordal skeleton (middle), which
    lead to a consistent extension when each oriented into an AMO
    (right).}
  \label{fig:mpdagbuckets}
\end{figure}

Buckets are similar to the undirected chordal components that we
considered in the previous section, in that orienting each
bucket in an MPDAG without cycles or v-structures will
yield a consistent extension. However, a bucket may already
contain directed edges. An illustration is
given in Fig.~\ref{fig:mpdagbuckets}.

\begin{fact}[\citet{WienobstExtendability2021}] \label{fact:mpdagext}
  An MPDAG is extendable (i.\,e., has a consistent extension) iff the
  skeleton of every bucket
  is chordal. An extension can be computed in time $O(n+m)$.
\end{fact}

An MPDAG can be extended by running a modified MCS for
each bucket $B$
in the following way: The graph traversal is performed
on the skeleton of $B$ with the restriction that only vertices
in the highest-label set $S$ that have no unvisited parent in $B$
are considered. These are the vertices $x$ with $\Pa(x) \setminus \tau = \emptyset$,
given that $\tau$ contains the visited vertices,
and we denote the set of such
vertices by $S^+ \subseteq S$. This way, the MCS conforms to the
background edges.

With these insights, an analogue modification of
Algorithm~\ref{alg:mcs} suggests itself to enumerate \emph{all} AMOs of a bucket:
Perform the algorithm on the skeleton of the bucket and only consider
vertices in $S^+$.
%As before, successively, only a single set of
%\emph{connected} vertices in $S^+$ is considered.

\begin{lemma}\label{lemma:pdagconn}
  Let $B$ be a bucket and $\tau$ be a sequence of visited vertices with
  $|\tau| = k < n$ produced by the modified MCS using~$S^+$. Then it
  holds that:
  \begin{enumerate}
  \item\label{lemma:pdagconn:a} If $x,y \in S^+$ are connected in
    $\text{skel}(B[V \setminus \tau])$, the
    set of AMOs produced by choosing
    $x$ next is disjoint from the set produced by choosing $y$ next.
  \item\label{lemma:pdagconn:b} If $x,y \in S^+$ are unconnected in $\text{skel}(B[V \setminus \tau])$, any
    AMO produced by choosing $y$ as the next vertex can also be
    produced by choosing a vertex in $S^+$ connected to $x$ in $\text{skel}(B[V \setminus \tau])$ next.
  \end{enumerate}
\end{lemma}

\begin{proof}
  For item~\ref{lemma:pdagconn:a} consider the shortest path
  between~$x$ and~$y$ and assume that it contains directed edges (if
  not the argument of Lemma~\ref{lemma:amoconn} applies). Then
  $x$ or $y$ have an incoming edge due to the non-applicability of
  the first Meek rule in the original bucket $B$. This violates the assumption that
  $x,y\in S^+$, meaning that
  $\tau_1$ and $\tau_2$ imply different AMOs.
  For item~\ref{lemma:pdagconn:b}  the same argument as in
  Lemma~\ref{lemma:amoconn} holds.
\end{proof}

Reachability can again be tested in a simplified way:
\begin{lemma}\label{lemma:pdagfastreach}
  Given a bucket $B$ and a sequence of visited vertices
  $\tau$ produced by the modified MCS using~$S^+$.
  Vertices $x,y \in S^+$ are connected in $B[V
  \setminus \tau]$ iff they are connected in $B[S^+]$.
\end{lemma}

Using these results, we can show that:

\begin{theorem}
  There is an algorithm that enumerates all
  AMOs of a given bucket $B$ with worst-case delay $O(n+m)$.
\end{theorem}

\begin{proof}[Proof sketch]
  Consider the just sketched algorithm, i.\,e., which proceeds as
  Algorithm~\ref{alg:mcs} for the skeleton of $B$ with the
  modification of choosing vertices and performing reachability with
  regard to $S^+$ (the algorithm is given explicitly in
  Appendix~\ref{subsection:appendixpdag}). 
  By using $S^+$, the resulting AMOs conform with the
  directed edges in the bucket and due to Lemma~\ref{lemma:pdagconn}
  and~\ref{lemma:pdagfastreach} and by
  similar arguments as for Theorem~\ref{theorem:mcsenumcorrectness} every
  such AMO is output exactly once.
  The linear-time delay follows as before, notably, $S^+$ can be
  efficiently maintained by storing the in-degree of each vertex.
\end{proof}

Using Fact~\ref{fact:mpdagext}, the result for buckets immediately generalizes to MPDAGs.
\begin{corollary} \label{corollary:enummpdag}
  There is an algorithm that enumerates all
  consistent extensions of a given MPDAG with linear-time delay.
\end{corollary}

The matter for PDAGs is similar as they can be
maximally oriented into an equivalent MPDAG 
%(which has
%the same set of consistent extensions) 
by Meek's rules.

\begin{theorem}
  There is an algorithm that enumerates all
  consistent extensions of a given PDAG with linear-time delay after an
  initialization step of time $O(n^3)$.
\end{theorem}

\begin{proof}
  The graph is initially transformed into its MPDAG. This is possible in time
  $O(n^3)$ as shown in~\cite{WienobstExtendability2021}.
  Afterward, apply the algorithm from Corollary~\ref{corollary:enummpdag}.
\end{proof}

\citet{WienobstExtendability2021} showed that the initialization step
of maximally orienting a
PDAG is likely not possible in linear time. However, using a finer
complexity analysis, it can be performed in time $O(dm)$,
where $d$ is the \emph{degeneracy} of the input's skeleton, which implies
linear-time on many natural graph classes such as planar graphs,
bounded-degree and bounded-treewidth graphs.

\section{Another Approach for Enumerating Markov Equivalent
  DAGs} \label{section:anotherapproach}

\begin{figure*}
  \centering
  \begin{tikzpicture}[xscale = 0.85, yscale=0.85]
    \node (a) at (-11,1) {$a$};
    \node (b) at (-10,1) {$b$};
    \node (c) at (-8,.5) {$c$};
    \node (d) at (-7,.5) {$d$};
    \node (e) at (-11,0) {$e$};
    \node (f) at (-10,0) {$f$};
    \node (g) at (-9,0) {$g$};
    \node (h) at (-8,-.5) {$h$};
    \node (i) at (-7,-.5) {$i$};
    \node (j) at (-11,-1) {$j$};
    \node (k) at (-10,-1) {$k$};
    \node (l) at (-9,1.5) {CPDAG};

    \graph[use existing nodes, edges = {edge}] {
      a -- b -- g;
      g -- c -- d;
      e -- f -- g;
      g -- h -- i;
      j -- k -- g;
    };

    \draw (-3.6,0) ellipse (2 and 1.5);
    
    \node[inner sep = 0.5] (a) at (-5,.75) {$a$};
    \node[inner sep = 0.5] (b) at (-4.25,.75) {$b$};
    \node[inner sep = 0.5] (c) at (-2.75,.5) {$c$};
    \node[inner sep = 0.5] (d) at (-2,.5) {$d$};
    \node[inner sep = 0.5] (e) at (-5,0) {$e$};
    \node[inner sep = 0.5] (f) at (-4.25,0) {$f$};
    \node[inner sep = 0.5] (g) at (-3.5,0) {$g$};
    \node[inner sep = 0.5] (h) at (-2.75,-.5) {$h$};
    \node[inner sep = 0.5] (i) at (-2,-.5) {$i$};
    \node[inner sep = 0.5] (j) at (-5,-.75) {$j$};
    \node[inner sep = 0.5] (k) at (-4.25,-.75) {$k$};

    \graph[use existing nodes, edges = {arc}] {
      a -- b -- g;
      g -- c -- d;
      g -- f -- e;
      g -- h -- i;
      g -- k -- j;
    };

    \draw[dashed] (-1.6,0) -- (0,1);
    
    \node[draw,circle] (a) at (0,1) {};
    \node[draw,circle] (b) at (1,1) {};
    \node[draw,circle] (c) at (3,.5) {};
    \node[draw,circle] (d) at (4,.5) {};
    \node[draw,circle] (e) at (0,0) {};
    \node[draw,circle] (f) at (1,0) {};
    \node[draw,circle] (g) at (2,0) {};
    \node[draw,circle] (h) at (3,-.5) {};
    \node[draw,circle] (i) at (4,-.5) {};
    \node[draw,circle] (j) at (0,-1) {};
    \node[draw,circle] (k) at (1,-1) {};
    \node (l) at (2,1.5) {MEC};

    \graph[use existing nodes, edges = {edge}] {
      g -- b -- a;
      g -- c -- d;
      g -- f -- e;
      g -- h -- i;
      g -- k -- j;
    };

    \draw (-0.8,-1.3) to[out=140,in=240] (-0.6,0.25) to[out=50,in=-40]
    (-0.6,1.1) to[out=130,in=170] (3,0.9) to[out=-10,in=120] (4.8,1.3)
    to[out=-60,in=75] (4.2,-1.2) to[out=-100,in=40] (2,-1)
    to[out=220,in=-40] (-0.8,-1.3);

    \draw (7.4,0) ellipse (2 and 1.5);
    
    \node[inner sep = 0.5] (a) at (6,.75) {$a$};
    \node[inner sep = 0.5] (b) at (6.75,.75) {$b$};
    \node[inner sep = 0.5] (c) at (8.25,.5) {$c$};
    \node[inner sep = 0.5] (d) at (9,.5) {$d$};
    \node[inner sep = 0.5] (e) at (6,0) {$e$};
    \node[inner sep = 0.5] (f) at (6.75,0) {$f$};
    \node[inner sep = 0.5] (g) at (7.5,0) {$g$};
    \node[inner sep = 0.5] (h) at (8.25,-.5) {$h$};
    \node[inner sep = 0.5] (i) at (9,-.5) {$i$};
    \node[inner sep = 0.5] (j) at (6,-.75) {$j$};
    \node[inner sep = 0.5] (k) at (6.75,-.75) {$k$};

    \graph[use existing nodes, edges = {arc}] {
      g -- b -- a;
      g -- c -- d;
      g -- f -- e;
      g -- h -- i;
      g -- k -- j;
    };

    \draw[dashed] (5.4,0) -- (2,0);
  \end{tikzpicture}
  \caption{An example that has no sequence of SHD two that enumerates all
    Markov equivalent DAGs. Two DAGs in the MEC are connected by an edge if they can be
    transformed into each other by a single edge reversal. For trees,
    the resulting topology coincides with the one of the CPDAG, each DAG in the
    MEC can be represented by its unique source vertex. During the
    enumeration, the DAG in the center can be used only once,
    which makes it impossible to list all ``leaf DAGs'' when allowing only
    for distance at most two.}
  \label{fig:counterexample}
\end{figure*}
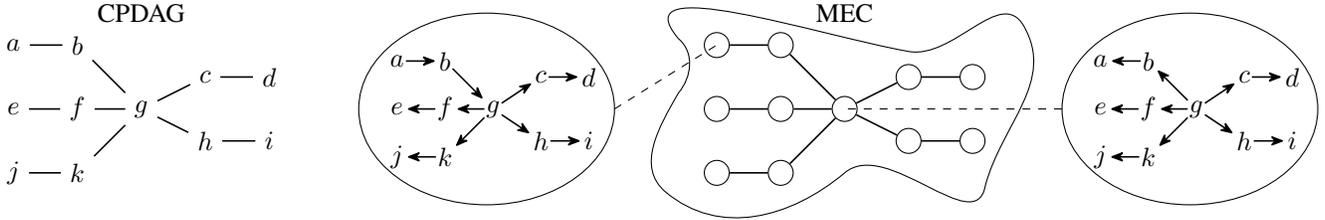
The results of the previous sections settle the worst-case complexity
of enumerating the members of an MEC (at least if
every DAG is output separately, a run-time of $o(n+m)$ is not
achievable as this would be less than the size of the graph). In this
section, we complement these
results with an enumeration sequence of small changes between
consecutive DAGs. While having a worse delay, such sequences are more
natural from the causal perspective, with only a few edge orientations changing
at a time, and provide structural insights into Markov
equivalence itself.
In more detail, we show that all graphs in an MEC
can be enumerated in a sequence such that every two consecutive DAGs
have structural Hamming distance (SHD) at most three. Our results are based on
the following characterization of Markov equivalence, which is also
the basis for \textsc{chickering-enum} (see also Appendix~\ref{appendix:chickeringsubsection}):
\begin{fact}[\citet{chickering1995transformational}] \label{fact:chickering}
  % Chickering: https://arxiv.org/pdf/1302.4938.pdf
  For any two Markov equivalent DAGs $D$ and $D'$ there is a
  sequence of Markov equivalent DAGs
  $\langle D = D_{1}, \dots, D_{k} = D' \rangle$ such that
  $D_{i}$ and $D_{i+1}$ have SHD one.
\end{fact}
% we could do without introducing covered edges, for saving space
%The edges that get reversed in this process are called
%\emph{covered edges}.
The statement also holds for two consistent
extensions of a \emph{PDAG} in the sense that all intermediate DAGs
are also consistent extensions of this PDAG.

\begin{corollary}\label{corollary:pdagchickering}
  For any two consistent extensions $C$ and $C'$ of PDAG $G$ there is
  a sequence of consistent extensions of $G$
  $\langle C = D_{1}, \dots, D_{k} = C' \rangle$ such that
  $D_{i}$ and $D_{i+1}$ have SHD one.
\end{corollary}

\begin{proof} 
  Only differing edges between $C$ and $C'$ are reversed in the
  constructive proof of~\citet{chickering1995transformational}. Hence,
  all background edges stay fixed during the transformation.
\end{proof}

This means that it is possible to go from one DAG to
another with only single edge reversals for CPDAGs as well as for
PDAGs and MPDAGs.
The task we are trying to solve, however, is to enumerate \emph{all}
members of an MEC, meaning
the goal is to find a sequence in which
\emph{every} DAG occurs exactly once. It can be shown
that such a sequence with SHD at most one does
indeed \emph{not} exist. Fig.~\ref{fig:counterexample} provides
an example that does not even allow a sequence of SHD \emph{two.}

However, if we permit \emph{three} edge reversals between consecutive
DAGs we can always find such a sequence:
\begin{theorem}\label{theorem:dist3seq}
  Every MEC % of DAGs
  can be represented as sequence $\langle D_1,
  D_2, \dots \rangle$ of Markov equivalent DAGs such that $D_i$ and
  $D_{i+1}$ have SHD at most three.
\end{theorem}

\begin{proof}
  For a constructive, proof consider the graph that contains all DAGs
  in the MEC as
  nodes.\footnote{We use the term node instead of vertex here to avoid
    confusion with the vertices of the DAGs.} In that graph connect
  two nodes with an edge if the DAGs can be transformed into each other
  by a single edge reversal (hence, these have SHD one). By
  Fact~\ref{fact:chickering} the graph is connected.

  Every connected graph has a sequence $\langle p_1, p_2, \dots
  \rangle$ that
  contains every node exactly once such that the distance between
  consecutive nodes is at most three.\footnote{The authors became
    aware of this graph property due to a problem posed by Jorke de Vlas in
    the annual programming contest BAPC (Problem H at BAPC 2021: \url{https://2021.bapc.eu/}).} This sequence can be
  constructed by performing a depth-first-search (DFS) starting at an
  arbitrary node $r$ and appending
  nodes with an even distance from $r$ in the DFS tree when they
  are discovered and nodes with an odd distance from $r$ when
  they are fully processed (essentially mixing pre- and post-order depending
  on the layer of the DFS tree). The SHD between two output
  nodes is never larger than three: When going down the DFS tree,
  every second node is output, when going up (after last outputting
  in odd layer $i$) the node in layer $i-2$ is output after it is
  finished. Hence, if it has no unvisited neighbors, the SHD is
  two. If it does, one of
  these gets explored and, hence, immediately output as it is in even
  layer $i-1$. In this case, the SHD to the last output is three.
\end{proof}

Due to Corollary~\ref{corollary:pdagchickering}, 
%this result can
%naturally be adapted to consistent extensions of a PDAG:
this result generalizes to PDAGs:

\begin{corollary}
  The consistent extensions of PDAG $G$ can be represented as sequence $\langle D_1,
  D_2, \dots \rangle$ of consistent extensions of $G$ with SHD at most three.
\end{corollary}

%These results are tight in the sense that a sequence with less edge
%(reversals between outputs does not
%necessarily exist as the counterexample in
%Fig.~\ref{fig:counterexample} illustrates. 
We note that 
\textsc{meek-enum},
\textsc{chickering-enum} and \textsc{mcs-enum} do not have this property.
%that the SHD between two outputs is at most three.

\begin{lemma}
  Sequences of DAGs produced by \textsc{meek-enum}
  may contain consecutive DAGs with
  SHD larger than three.
\end{lemma}

\begin{proof}
  Consider the CPDAG shown in Fig.~\ref{fig:counterexample}. Since
  \textsc{meek-enum} has no preferences on the edge it orients first, it may
  start with the edge $a \rightarrow b$. All other edge directions
  would then follow from the first Meek rule
  yielding the output DAG shown in Fig.~\ref{fig:counterexample}.
  The orientation $a \leftarrow b$ is tried afterward, which would
  result in no further directed edges. Then, assume the next undirected edge
  picked by the algorithm is the ones between $h$ and $i$. It may be
  oriented as $h \leftarrow i$ yielding a DAG with SHD 4 to the previously output DAG. 
\end{proof}

Similar arguments
hold for \textsc{chickering-enum} and \textsc{mcs-enum}, the former could end up in a state where the
only DAGs left are the one with edge $a \rightarrow b$ and the
one with $h \leftarrow i$ and the latter could start with vertex $a$
and afterward choose $i$ as the first vertex~--~yielding again the same DAGs.

\begin{corollary}
  Sequences of DAGs produced by \textsc{chickering-enum} and \textsc{mcs-enum}
  may contain consecutive DAGs with
  structural Hamming distance larger than three.
\end{corollary}

% \footnote{It is of course possible to
%   implement the Meek rule algorithm such that the next undirected
%   edge is close to the last chosen one (Algorithm~\ref{alg:mcs} may
%   be adapted similarly), but this would still not give similar
%   guarantees as Theorem~\ref{theorem:dist3seq}.}

Computationally, our results do not imply a better bound on
the delay in producing the sequence from
Theorem~\ref{theorem:dist3seq} and we leave this as an open
problem. The constructive algorithm (which we call \textsc{shd3-enum})
given in the proof of
Theorem~\ref{theorem:dist3seq} behaves similar to
\textsc{chickering-enum} and has delay $O(m^2)$ as every DAG
may have $m$ neighbors and we have to check for
each of them whether they
were already visited (between two outputs a constant number of
recursive calls are handled; see Appendix~\ref{subsection:appendixleq3}).
%for a more detailed analysis). 
It seems unlikely that this can be improved
without further structural insights.

Lastly, we remark that the same idea can also be used in the more
general setting of enumerating maximal ancestral graphs (MAGs) without
selection bias, which are causal models allowing for latent
confounders and for which a
similar transformational characterization
exists~\cite{zhang2005transformational}.
A brief introduction to MAGs and a more detailed analysis are given in
Appendix~\ref{section:appendixmags}.
\begin{corollary}\label{corollary:mags}
  Every MEC of MAGs without selection bias
  can be represented as sequence $\langle M_1,
  M_2, \dots \rangle$ of Markov equivalent MAGs with SHD at most three.
\end{corollary}

% This illustrates the generality of the
% approach in Theorem~\ref{theorem:dist3seq}.

In contrast, approaches such as \textsc{meek-enum}
do not exist for MAGs as analogue rules proposed
by~\citet{zhang2008completeness}
only complete the graph in case the edge marks are inferred by
observational data. If edge marks are chosen for the sake of
enumeration, these rules are not known to be 
complete. Generally, it
is an interesting direction for future work to investigate the
computational aspects of the enumeration of MECs of MAGs.

\section{Experiments}
In addition to the theoretical results, we also show that
\textsc{mcs-enum} and its generalizations are practically
implementable and significantly faster than previously used algorithms.

\begin{figure}
  \centering
  \begin{tikzpicture}[xscale=0.8, yscale=0.5]

    \begin{scope}[yscale=0.75]
      \plot[0.2cm]{
        0.00590,
        0.00859,
        0.01367,
        0.02314,
        0.04173,
        0.07717,
        0.14963
      }{ba.pine}{\textsc{mcs}}
      
      \plot{
        0.10779,
        0.26160,
        0.69816,
        1.41159,
        3.00775,
        5.97319,
        12.81452
      }{ba.blue}{\textsc{meek}}
      
      \plot[0.25cm]{
        0.02418,
        0.06906,
        0.25004,
        0.43234,
        0.98316,
        2.20901,
        6.59289
      }{ba.orange}{\textsc{shd3}}
      
      \plot[-0.25cm]{
        0.02462,
        0.06949,
        0.24948,
        0.43453,
        0.97970,
        2.20470,
        6.54484
      }{ba.violet}{\textsc{chickering}}
      
      \draw[semithick, ->, >={[round]Stealth}] (0,0) -- (0, 13)
      node[above] {Delay in ms};
      \foreach \y in {1,2,...,12} {
        \draw (0,\y) -- (-.25,\y) node[left] {\small\y};
      }
      
    \end{scope}

    \draw[semithick, ->, >={[round]Stealth}] (0,0) -- (9,0);
    \foreach [count=\x] \label in {16, 32, 64, 128, 256, 512, 1024}{
      \draw (\x,0) -- (\x, -0.25) node[below] {\small\label};
    }
    \node at (4,-1.5) {Number of Vertices};
  \end{tikzpicture}
  \caption{Average delay in milliseconds for enumerating the
    AMOs of random chordal graphs with $m = 3\cdot n$ edges. We
    compare the algorithms \textsc{meek-enum}, \textsc{chickering-enum}, \textsc{mcs-enum} and \textsc{shd3-enum}.
  }
  \label{fig:experiments}
\end{figure}
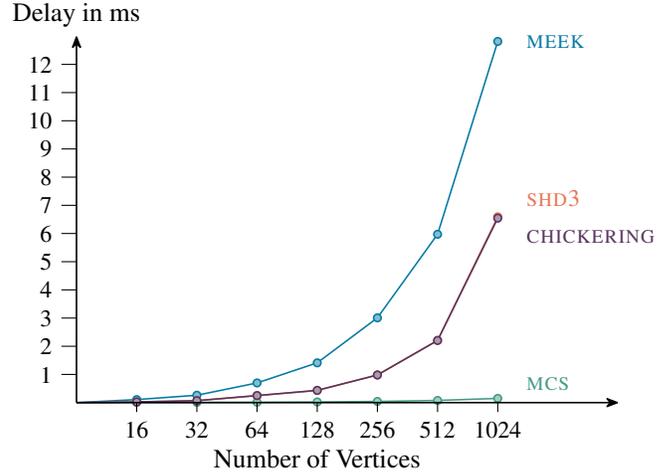

In Fig.~\ref{fig:experiments}, we compare the average delay of the
four approaches (\textsc{meek-enum}, \textsc{chickering-enum},
\textsc{mcs-enum}, \textsc{shd3-enum}) implemented in
Julia~\cite{bezanson2017julia}.
For each instance, the programs
were terminated after two minutes if the enumeration was not completed.
As the enumeration problem reduces to listing the AMOs of a chordal
graph (as shown in Fact~\ref{fact:reduction}), we consider as instances
undirected graphs generated
by randomly inserting edges, which do not violate chordality, until a
graph with $3\cdot n$ edges is reached. Note that these
instances are all CPDAGs, just fully undirected ones, and thus all
approaches can be applied to this setting.
In Appendix~\ref{appendix:experiments}, we also compare the
results for CPDAGs \emph{with} directed edges as well as for PDAGs,
which both lead to
very similar results. Moreover, we discuss the distribution of the
delay for the various algorithms.\footnote{The implementations
of the algorithms are available at \\ \url{https://github.com/mwien/mec-enum}.}

The results clearly show that \textsc{mcs-enum} is by far the
fastest among the algorithms. This is mainly due to the fact that the
other algorithms always incur a cost of at least $\Omega(n+m)$, whenever a single
edge is (re-)oriented. \textsc{meek-enum} needs to apply the four completion
rules, whereas \textsc{chickering-enum} and \textsc{shd3-enum}
require checking whether the resulting DAG was already output (which
might often be the case). Still, the latter algorithms are
significantly faster than \textsc{meek-enum} (at the cost of higher
memory demand), and notably have both very similar delay
(showing that the enumeration with SHD at most three  gives mainly
structural insights into Markov equivalence and has in itself no
computational advantage).

\section{Conclusion}
We have given the first formal and exhaustive treatment of the
fundamental problem of enumerating Markov equivalent DAGs. Our main
results are twofold: (i) we significantly improve the run-time of
enumeration by giving the first linear-time delay algorithm, which is
also practically effective and (ii) we
give structural insights into Markov equivalence by constructing an
enumeration sequence with minimal distance between
successive graphs. The concepts for (ii) are so general that they directly apply
to MAGs without selection bias as well.

As an open problem, it remains to find more efficient enumeration
algorithms for MAGs, where, currently, approaches in the spirit of both
\textsc{meek-enum} and
\textsc{mcs-enum} cannot be applied,
because similar structure does not exist for Markov equivalence of
MAGs or, at least, is not known.

\section*{Acknowledgments}
The research of Malte Luttermann was partly supported by the Medical Cause
and Effects Analysis (MCEA) project. The work of the last author was partly supported by the Deutsche Forschungsgemeinschaft (DFG) grant 471183316.

\bibliography{main}

\begin{thebibliography}{28}
\providecommand{\natexlab}[1]{#1}

\bibitem[{Albert and Barabási(2002)}]{BarabasiAlbert2002}
Albert, R.; and Barabási, A.-L. 2002.
\newblock Statistical mechanics of complex networks.
\newblock \emph{Reviews of Modern Physics}, 74(1): 47--97.

\bibitem[{Andersson, Madigan, and Perlman(1997)}]{andersson1997characterization}
Andersson, S.~A.; Madigan, D.; and Perlman, M.~D. 1997.
\newblock A characterization of {Markov} equivalence classes for acyclic digraphs.
\newblock \emph{The Annals of Statistics}, 25(2): 505--541.

\bibitem[{Bezanson et~al.(2017)Bezanson, Edelman, Karpinski, and Shah}]{bezanson2017julia}
Bezanson, J.; Edelman, A.; Karpinski, S.; and Shah, V.~B. 2017.
\newblock Julia: A fresh approach to numerical computing.
\newblock \emph{SIAM review}, 59(1): 65--98.

\bibitem[{Chen, Choi, and Darwiche(2016)}]{chen2016enumerating}
Chen, E. Y.-J.; Choi, A.~C.; and Darwiche, A. 2016.
\newblock Enumerating equivalence classes of Bayesian networks using EC graphs.
\newblock In \emph{Artificial Intelligence and Statistics}, 591--599. PMLR.

\bibitem[{Chickering(1995)}]{chickering1995transformational}
Chickering, D.~M. 1995.
\newblock A transformational characterization of equivalent {Bayesian} network structures.
\newblock In \emph{Proceedings of the 11th Conference on Uncertainty in Artificial Intelligence, {UAI} '95}, 87--98.

\bibitem[{Chickering(2002)}]{chickering2002optimal}
Chickering, D.~M. 2002.
\newblock Optimal Structure Identification With Greedy Search.
\newblock \emph{Journal of Machine Learning Research}, 3: 507--554.

\bibitem[{Elwert(2013)}]{Elwert2013}
Elwert, F. 2013.
\newblock Graphical Causal Models.
\newblock In \emph{Handbook of Causal Analysis for Social Research}, Handbooks of Sociology and Social Research, 245--273. Springer.

\bibitem[{Frydenberg(1990)}]{frydenberg1990chain}
Frydenberg, M. 1990.
\newblock The chain graph {Markov} property.
\newblock \emph{Scandinavian Journal of Statistics}, 333--353.

\bibitem[{Gillispie and Lemieux(2001)}]{gillispie2001enumerating}
Gillispie, S.~B.; and Lemieux, C. 2001.
\newblock Enumerating Markov Equivalence Classes of Acyclic Digraph Models.
\newblock In \emph{Proceedings of the 17th Conference in Uncertainty in Artificial Intelligence, {UAI} '01}, 171--177.

\bibitem[{Hauser and B{\"{u}}hlmann(2012)}]{hauser2012characterization}
Hauser, A.; and B{\"{u}}hlmann, P. 2012.
\newblock Characterization and Greedy Learning of Interventional {Markov} Equivalence Classes of Directed Acyclic Graphs.
\newblock \emph{Journal of Machine Learning Research}, 13: 2409--2464.

\bibitem[{He, Jia, and Yu(2015)}]{He2015}
He, Y.; Jia, J.; and Yu, B. 2015.
\newblock Counting and Exploring Sizes of {Markov} Equivalence Classes of Directed Acyclic Graphs.
\newblock \emph{Journal of Machine Learning Research}, 16(79): 2589--2609.

\bibitem[{He and Geng(2008)}]{he2008active}
He, Y.-B.; and Geng, Z. 2008.
\newblock Active learning of causal networks with intervention experiments and optimal designs.
\newblock \emph{Journal of Machine Learning Research}, 9(Nov): 2523--2547.

\bibitem[{Kalisch et~al.(2012)Kalisch, M{\"a}chler, Colombo, Maathuis, and B{\"u}hlmann}]{kalisch2012causal}
Kalisch, M.; M{\"a}chler, M.; Colombo, D.; Maathuis, M.~H.; and B{\"u}hlmann, P. 2012.
\newblock Causal inference using graphical models with the {R} package pcalg.
\newblock \emph{Journal of statistical software}, 47: 1--26.

\bibitem[{Koller and Friedman(2009)}]{koller2009probabilistic}
Koller, D.; and Friedman, N. 2009.
\newblock \emph{Probabilistic Graphical Models - Principles and Techniques}.
\newblock {MIT} Press.
\newblock ISBN 978-0-262-01319-2.

\bibitem[{Maathuis, Kalisch, and B{\"u}hlmann(2009)}]{maathuis2009estimating}
Maathuis, M.~H.; Kalisch, M.; and B{\"u}hlmann, P. 2009.
\newblock Estimating High-Dimensional Intervention Effects from Observational Data.
\newblock \emph{The Annals of Statistics}, 37(6A): 3133--3164.

\bibitem[{Meek(1995)}]{Meek1995}
Meek, C. 1995.
\newblock Causal Inference and Causal Explanation with Background Knowledge.
\newblock In \emph{Proceedings of the 11th Conference on Uncertainty in Artificial Intelligence, {UAI}~'95}, 403--410.

\bibitem[{Pearl(2009)}]{pearl2009causality}
Pearl, J. 2009.
\newblock \emph{Causality}.
\newblock {Cambridge University Press}.
\newblock ISBN 978-0521895606.

\bibitem[{Richardson and Spirtes(2002)}]{richardson2002ancestral}
Richardson, T.; and Spirtes, P. 2002.
\newblock Ancestral graph Markov models.
\newblock \emph{The Annals of Statistics}, 30(4): 962--1030.

\bibitem[{Rothman et~al.(2008)Rothman, Greenland, Lash et~al.}]{rothman2008modern}
Rothman, K.~J.; Greenland, S.; Lash, T.~L.; et~al. 2008.
\newblock \emph{Modern epidemiology}, volume~3.
\newblock Wolters Kluwer Health/Lippincott Williams \& Wilkins Philadelphia.

\bibitem[{Spirtes, Glymour, and Scheines(2000)}]{spirtes2000causation}
Spirtes, P.; Glymour, C.; and Scheines, R. 2000.
\newblock \emph{Causation, Prediction, and Search, Second Edition}.
\newblock {MIT} Press.
\newblock ISBN 978-0-262-19440-2.

\bibitem[{Squires(2018)}]{squires2018causaldag}
Squires, C. 2018.
\newblock \emph{{\texttt{causaldag}: creation, manipulation, and learning of causal models}}.

\bibitem[{Steinsky(2003)}]{steinsky2003enumeration}
Steinsky, B. 2003.
\newblock Enumeration of labelled chain graphs and labelled essential directed acyclic graphs.
\newblock \emph{Discrete mathematics}, 270(1-3): 267--278.

\bibitem[{Tarjan and Yannakakis(1984)}]{tarjan1984simple}
Tarjan, R.~E.; and Yannakakis, M. 1984.
\newblock Simple linear-time algorithms to test chordality of graphs, test acyclicity of hypergraphs, and selectively reduce acyclic hypergraphs.
\newblock \emph{SIAM Journal on computing}, 13(3): 566--579.

\bibitem[{Verma and Pearl(1990)}]{verma1990equivalence}
Verma, T.; and Pearl, J. 1990.
\newblock Equivalence and Synthesis of Causal Models.
\newblock In \emph{Proceedings of the 6th Conference on Uncertainty in Artificial Intelligence, UAI'90}, 255--270.

\bibitem[{Wien{\"o}bst, Bannach, and Li\'{s}kiewicz(2021{\natexlab{a}})}]{WienobstExtendability2021}
Wien{\"o}bst, M.; Bannach, M.; and Li\'{s}kiewicz, M. 2021{\natexlab{a}}.
\newblock Extendability of Causal Graphical Models: Algorithms and Computational Complexity.
\newblock In \emph{Proceedings of the 37th Conference in Uncertainty in Artificial Intelligence, {UAI} '21}. {AUAI} Press.

\bibitem[{Wien{\"o}bst, Bannach, and Li\'{s}kiewicz(2021{\natexlab{b}})}]{wienobst2021counting}
Wien{\"o}bst, M.; Bannach, M.; and Li\'{s}kiewicz, M. 2021{\natexlab{b}}.
\newblock Polynomial-Time Algorithms for Counting and Sampling {Markov} Equivalent {DAG}s.
\newblock In \emph{Proceedings of the AAAI Conference on Artificial Intelligence, {AAAI 2021}}, volume~35, 12198--12206. {AAAI} Press.

\bibitem[{Zhang(2008)}]{zhang2008completeness}
Zhang, J. 2008.
\newblock On the completeness of orientation rules for causal discovery in the presence of latent confounders and selection bias.
\newblock \emph{Artificial Intelligence}, 172(16-17): 1873--1896.

\bibitem[{Zhang and Spirtes(2005)}]{zhang2005transformational}
Zhang, J.; and Spirtes, P. 2005.
\newblock A transformational characterization of Markov equivalence for directed acyclic graphs with latent variables.
\newblock In \emph{Proceedings of the 21st Conference on Uncertainty in Artificial Intelligence, {UAI} '05}, 667--674.

\end{thebibliography}

\clearpage

\appendix

\begin{strip}
  \centering
  \textbf{\huge Appendix}
  \vspace*{1cm}
\end{strip}

\section{Formal Description of Related Algorithms}
\label{appendix:related}
\subsection{Enumerating Markov Equivalent DAGs Based on Meek's Rules}
\label{appendix:meeksubsection}
\citet{Meek1995} gave a complete set of four rules (R1 - R4), which,
when applied repeatedly, transform a PDAG into its maximal
orientation, i.\,e., orient all undirected edges, which are fixed in the
DAGs represented by the PDAG. Graphs completed under the Meek rules
are also called \emph{Maximally oriented PDAGs} (MPDAGs).

\begin{figure}
  \begin{tikzpicture}[scale=0.85]
    \node (a) at (-1,0) {$a$};
    \node (b) at (-1,-1) {$b$};
    \node (c) at (0,-1) {$c$}; 
    \draw [arc](a) -- (b);
    \draw [-](b) -- (c);
    
    \node at (0.3+0.5,-0.5) {$\Rightarrow$};
    \node at (0.3+0.5, 0.25) {\textbf{R1}};
    
    \node (a) at (0.6+1,0) {$a$};
    \node (b) at (0.6+1,-1) {$b$};
    \node (c) at (0.6+2,-1) {$c$}; 
    \draw [arc](a) -- (b);
    \draw [arc](b) -- (c);

    \node (a) at (1.3+3,0) {$a$};
    \node (b) at (1.3+3,-1) {$b$};
    \node (c) at (1.3+4,-1) {$c$}; 
    \draw [arc](a) -- (b);
    \draw [arc](b) -- (c);
    \draw [-] (a) edge (c);
    
    \node at (1.3+0.3+4.5,-0.5) {$\Rightarrow$};
    \node at (1.3+0.3+4.5, 0.25) {\textbf{R2}};
    
    \node (a) at (1.3+0.6+5,0) {$a$};
    \node (b) at (1.3+0.6+5,-1) {$b$};
    \node (c) at (1.3+0.6+6,-1) {$c$};
    \draw [arc](a) -- (b);
    \draw [arc](b) -- (c);
    \draw [arc] (a) edge (c);
    
    \node (a) at (-1,-2) {$a$};
    \node (d) at (0,-2) {$d$};
    \node (b) at (-1,-3) {$b$};
    \node (c) at (0,-3) {$c$}; 
    \draw [-](a) -- (b);
    \draw [-](a) -- (d);
    \draw [-](a) -- (c);
    \draw [arc](b) -- (c);
    \draw [arc](d) -- (c);
    
    \node at (0.3+0.5,-2.5) {$\Rightarrow$};
    \node at (0.3+0.5, -1.75) {\textbf{R3}};
    
    \node (a) at (0.6+1,-2) {$a$};
    \node (d) at (0.6+2,-2) {$d$};
    \node (b) at (0.6+1,-3) {$b$};
    \node (c) at (0.6+2,-3) {$c$}; 
    \draw [-](a) -- (b);
    \draw [-](a) -- (d);
    \draw [arc](a) -- (c);
    \draw [arc](b) -- (c);
    \draw [arc](d) -- (c);      
    
    \node (a) at (1.3+3, -2) {$a$};
    \node (d) at (1.3+4, -2) {$d$};
    \node (b) at (1.3+3, -3) {$b$};
    \node (c) at (1.3+4, -3) {$c$};
    \draw [-] (a) -- (b);
    \draw [-] (a) -- (c);
    \draw [-] (a) -- (d);
    \draw [arc] (d) -- (c);
    \draw [arc] (c) -- (b);
    
    \node at (1.3+0.3+4.5, -2.5) {$\Rightarrow$};
    \node at (1.3+0.3+4.5, -1.75) {\textbf{R4}};
    
    \node (a) at (1.3+0.6+5, -2) {$a$};
    \node (d) at (1.3+0.6+6, -2) {$d$};
    \node (b) at (1.3+0.6+5, -3) {$b$};
    \node (c) at (1.3+0.6+6, -3) {$c$};
    \draw [arc] (a) -- (b);
    \draw [-] (a) -- (c);
    \draw [-] (a) -- (d);
    \draw [arc] (d) -- (c);
    \draw [arc] (c) -- (b);      
  \end{tikzpicture}
  \caption{The four Meek rules that are used to characterize MPDAGs~\citep{Meek1995}.}
  \label{figure:meekRules}
\end{figure}
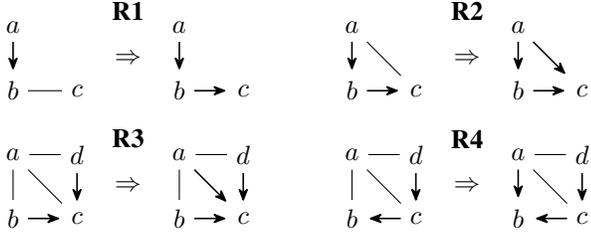

See R1 in Figure~\ref{figure:meekRules} as an example for the
application of the rules. The edge $b -
c$ necessarily has a fixed orientation because $a \rightarrow b
\leftarrow c$ and $a \rightarrow b \rightarrow c$ are in different
MECs (the first graph contains a v-structure, the latter does not) and
hence cannot be represented by the same PDAG. The single DAG
represented by the PDAG in this case is $a \rightarrow b \rightarrow
c$, the graph which does not introduce a new v-structure, and it is
immediately obtained after applying R1. In most cases some undirected
edges remain even after application of the Meek rules.

In the converse, a graph completed under the Meek rules has the
property that every undirected edge $x - y$ is \emph{not} fixed in the DAGs it
represents, i.\,e., there is a DAG in the class which contains $x \rightarrow y$
and one which contains $x \leftarrow y$.
Thus, the members can be enumerated by first orienting the edge as $x
\rightarrow y$ (and recursively orienting the remaining edges one-by-one, always
after completing the graph under the Meek rules to ensure that every
orientation yields a valid DAG) and afterward orienting the edge as $x
\leftarrow y$. This approach we call \textsc{meek-enum} is depicted in
Algorithm~\ref{alg:meek}.

\begin{algorithm}
  \caption{Enumeration algorithm of Markov equivalent DAGs based on
    Meek's rules (\textsc{meek-enum}).}
  \label{alg:meek}
  \DontPrintSemicolon
  \SetKwInOut{Input}{input}\SetKwInOut{Output}{output}
  \SetKwFor{Rep}{repeat}{}{end}
  \Input{A CPDAG $G = (V,E)$.}
  \Output{$[G]$.}
  \SetKwFunction{FEnum}{enumerate}
  \SetKwProg{Fn}{function}{}{end}
  \SetKwRepeat{Do}{do}{while}

  \BlankLine
  % \If{$G$ is extendable}{$\FEnum(G)$\;}
  % \BlankLine

  \Fn{\FEnum{$G$}}{
    $H := \text{maximal orientation of } G$\;
    $\mathrm{undir} := \{\{u,v\} \mid (u,v) \in E_H \wedge (v,u) \in E_H \}$\;
    \eIf{$\mathrm{undir}$ is empty}{
      Output $H$
    }{
      $\{u, v\} := \text{any element from } \mathrm{undir}$\;
      $E_H := E_H \setminus \{(u, v)\}$ \hspace*{0.1cm} \tcp{Orient $u \leftarrow v$}
      \FEnum{$H$}\;
      $E_H := E_H \cup \{(u, v)\}$ \hspace*{0.1cm} \tcp{Undo}
      \BlankLine
      $E_H := E_H \setminus \{(v, u)\}$ \hspace*{0.1cm} \tcp{Orient $u \rightarrow v$}
      \FEnum{$H$}\;
      $E_H := E_H \cup \{(v, u)\}$ \hspace*{0.1cm} \tcp{Undo}
    }
  }
\end{algorithm}

\textsc{meek-enum} has polynomial-delay, particularly as every
orientation leads to a valid DAG (there are no ``dead-ends'' in the
recursive search). However, it requires the application of Meek's rules in
every step, which amounts to significant effort to perform between two
consecutive outputs.
Naively estimated, checking whether one of the four Meek rules applies
takes time $O(n^4)$ and in a worst-case scenario this check has to be
performed multiple times as more and more edges get oriented
successively. This would lead to a polynomial time requirement of
large degree, even
for applying the Meek rules once. However, there have been recent
algorithmic improvements in this regard.

\begin{fact}[\citet{WienobstExtendability2021}] \label{fact:efficientmr}
  Given a PDAG $G$, Meek's rules can be applied exhaustively on $G$ in
  time $O(n^3)$.
\end{fact}

This gives a slightly better bound on the achievable delay, formally
stated in the following Theorem.
\begin{theorem}
  \textsc{meek-enum} can be implemented such that the MEC
  is enumerated with worst-case delay $O(m \cdot n^3)$.
\end{theorem}
\begin{proof}
  Clearly, every DAG in the MEC is output as every undirected edge is oriented
  in both directions (by soundness of the Meek rules every directed
  edge is actually fixed in the DAGs in the class).
  No DAG is output twice because in each step, an undirected edge is
  oriented and fixed, i.e., every recursive call receives a different
  graph as input. Nor is an invalid DAG output, because every
  undirected edge can be oriented in both directions leading to a
  valid DAG by completeness of the Meek rules.
  In the worst case, the algorithm needs $m$ recursive calls until a
  DAG is output, and a single call costs time $O(n^3)$ due
  to the application of Meek's rules by Fact~\ref{fact:efficientmr}.
  Hence, the delay between two outputs is bounded by $O(m \cdot n^3)$.
\end{proof}

The same result holds for PDAGs and MPDAGs without modifications.
We note that in the experiments, we decided to implement the Meek
rules in the ``naive'' way as this is typically used in
implementations, which makes for a more meaningful comparison.

\subsection{Enumerating Markov Equivalent DAGs Based on Chickering's
  Transformational Characterization}
\label{appendix:chickeringsubsection}
Another approach for enumerating Markov equivalent DAGs is build on
the characterization by~\citet{chickering1995transformational}, also
stated briefly in Section~\ref{section:anotherapproach}
above. It is based on the notion of \emph{covered edges}:
\begin{definition}[\citet{chickering1995transformational}]
  An edge $x \rightarrow y$ is called \emph{covered} if
  $\mathrm{Pa}(x) \cup \{x\} = \mathrm{Pa}(y)$.
\end{definition}

This allows the following characterization of Markov equivalence:
\begin{theorem}[\citet{chickering1995transformational}]\label{theorem:chickering}
  Let $D$ and $D'$ be two Markov equivalent DAGs. There exists a
  sequence of $\mathrm{shd}(D, D')$
  edge reversals in $D$ with the following properties:
  \begin{enumerate}
  \item Each edge reversed in $D$ is a covered edge.
  \item After each reversal, $D$ is a DAG Markov equivalent to $D'$.
  \item After all reversals, $D = D'$.
  \end{enumerate}
\end{theorem}

Hence, starting from a DAG $D$ we reach all DAGs in the same Markov
equivalence class by reversals of covered edges. This permits the
following enumeration approach, which we call \textsc{chickering-enum}.

\begin{algorithm}
    \caption{Enumeration algorithm of Markov
    equivalent DAGs based on Chickerings'
    characterization~\cite{chickering1995transformational} (\textsc{chickering-enum}).}
  \label{alg:chickering}
  \DontPrintSemicolon
  \SetKwInOut{Input}{input}\SetKwInOut{Output}{output}
  \SetKwFor{Rep}{repeat}{}{end}
  \Input{A CPDAG $G = (V,E)$.}
  \Output{$[G]$.}
  \SetKwFunction{FDFS}{enumerate}
  \SetKwFunction{FExtend}{extend}
  \SetKwProg{Fn}{function}{}{end}

  \BlankLine
  $D := $ any DAG in $[G]$ \;
  $\mathrm{vis} := \{D\}$ \;
  $\FDFS(D)$ \;
  \BlankLine
  
  \Fn{\FDFS{$D$, $\mathrm{vis}$}}{
    Output $D$ \;

    \ForEach{DAG $D'$ obtained by reversing a covered edge in $D$}{
      \If{$D' \not\in \mathrm{vis}$}{
        $\mathrm{vis} := \mathrm{vis} \cup \{D'\}$ \;
        \FDFS{$D'$, $\mathrm{vis}$}
      }
    }
  }
\end{algorithm}

Starting from $D$, similar to a depth-first-search (DFS), all
neighbors of $D$ (graphs with a single reversed covered edge) are
explored and this is continued recursively. Eventually all Markov equivalent DAGs are
reached by Theorem~\ref{theorem:chickering} above.
In order to not visit any DAG twice, it is necessary to store a set of
all visited DAGs.

\begin{theorem} \label{theorem:chickeringdelay}
  \textsc{chickering-enum} enumerates a Markov equivalence class and can be
  implemented with worst-case delay $O(m^3)$.
\end{theorem}

\begin{proof}
  Clearly, any DAG is output exactly once by
  Theorem~\ref{theorem:chickering} and storing
  all visited DAGs ensures not outputting any DAG multiple times.

  For the delay, it is first useful to analyze when the most steps
  between two outputs are necessary. This happens when the algorithm
  traverses back up the recursion tree (from a leaf), potentially up to the
  root, without outputting any new DAG and then down into another subtree.
  In principle, the recursion
  depth might be exponential and this would, in turn, lead to an
  exponential-time delay.

  However, it is possible to bound the recursion depth by $m$ (by
  Theorem~\ref{theorem:chickering} every edge is reversed at most
  once) and one will still find all DAGs in the MEC. It remains to
  estimate the cost at each recursion step. Going up the
  recursion tree, at each DAG, $O(m)$ neighbors might be checked
  whether they have been visited before, e.\,g., if they are in
  $\mathrm{vis}$, and all of them might indeed are.
  The check takes time $\Omega(m)$ per neighboring DAG $D'$ as
  the whole DAG has to be ``read'' at least once for lookup. If one
  uses a hash table for storing the DAGs, expected time $O(m)$ can be reached.

  So, in total, we have $O(m)$ recursion steps between two outputs,
  with at most $O(m)$ neighboring DAGs being considered, each of which consuming
  $O(m)$ time for lookup in $\mathrm{vis}$. This leads to a worst-case
  delay of $O(m^3)$.
\end{proof}

\subsection{Computing a Consistent Extension of a CPDAG based on the
  MCS Algorithm}
\label{appendix:mcssubsection}
In this subsection, we revisit the closely related problem of finding
an arbitrary DAG in an MEC represented by a CPDAG $G$ instead of listing all of them. This is known as the \emph{extension} task and it is well-known that it
is possible to perform it in linear-time~\cite{hauser2012characterization} for CPDAGs
using the so-called Lexicographic Breadth-First Search (LBFS)
algorithm, which is a graph traversal algorithm originally proposed
for testing chordality of a graph. For MPDAGs and PDAGs, this problem is discussed in depth in~\cite{WienobstExtendability2021}.

As we base our results on the Maximum Cardinality Search (MCS) algorithm, which has
been used less frequently in the causality community compared to LBFS, we give a brief
introduction. The MCS algorithm has been designed, as LBFS, for
testing chordality of a graph. It is a graph traversal algorithm,
which visits the vertices of an undirected graph in an order
representing\footnote{A linear ordering of the vertices represents the orientation of an undirected graph, in which edges are directed $x \rightarrow y$ if $x$ comes before $y$ in the ordering.} an AMO (\emph{acyclic moral orientation}, that is, an
orientation without directed cycles and v-structures) iff the
graph is chordal.\footnote{In the chordal graph literature the term
  \emph{perfect elimination ordering} is more commonly used in this
  setting. A linear-ordering of
  the vertices describes an AMO iff its reverse is a perfect
  elimination ordering. We will stick to AMOs in this paper to avoid
  confusion.}

The connection between AMOs and CPDAGs stems from the fact that the
undirected components of a CPDAG are chordal and replacing each by an AMO
leads to a DAG in the Markov equivalence
class. The reason for this is that a CPDAG is extended into a DAG
by orienting its undirected edges without creating a directed cycle or
a new v-structure. This is why AMOs are needed, and the undirected components
are chordal, because only chordal graphs have AMOs.

\begin{algorithm}[h!]
      \caption{Computing a DAG in the MEC represented by CPDAG $G$ in
        linear-time $O(n+m)$.}
  \label{alg:cpdagextension}
  \DontPrintSemicolon
  \SetKwInOut{Input}{input}\SetKwInOut{Output}{output}
  \SetKwFor{Rep}{repeat}{}{end}
  \Input{A CPDAG $G = (V,E)$.}
  \Output{DAG $D \in [G]$.}
  \SetKwFunction{FMCS}{mcs}
  \SetKwProg{Fn}{function}{}{end}

  \BlankLine
  $D := $ copy of $G$ \;
  $U := (V, \{ (u,v) \mid (u,v) \text{ and } (v,u) \in E \})$, i.\,e.,
  the graph with only the undirected edges of $G$ \; \label{line:startpre}
  \ForEach{connected component $C = (V_C, E_C)$ of $U$}{
    $\tau_C = \FMCS(C)$ \;
    Orient edges in $D[V_C]$ according to $\tau_C$, i.\,e., $x - y$ as
    $x \rightarrow y$ if $x$ comes before $y$ in $\tau_C$ \;
  } \label{line:endpre}

  \BlankLine
  
  \Fn{\FMCS{$C$} \label{line:startmcs}}{
    $A :=$ array of $n$ initially empty sets \; 
    $\tau :=$ empty list \;
    $A[0] := V$ \;

    \While{$|\tau| < |V_C|$}{
      $i :=  \text{highest index of non-empty set in } A$\;
      $v := \text{any vertex from } A[i]$ \; \label{line:choosenext}
      $\text{delete } v \text{ from } A[i]$ \; \label{line:starthandle}
      $\text{append } v \text{ to } \tau$ \;
      \ForEach{$w \in (\Ne(v) \setminus \tau)$}{
        $j := \text{index of set in } A \text{ that } w \text{ is
          in}$ \;
        $\text{delete } w \text{ from } A[j]$ \;
        $\text{insert } w \text{ in } A[j+1]$ \; \label{line:movehigher}
      } \label{line:endhandle}
    }
    \Return $\tau$ \;
  } \label{line:endmcs}
\end{algorithm}

We can now look at Algorithm~\ref{alg:cpdagextension}, which
implements this strategy, in more detail. We discussed the first part
(lines~\ref{line:startpre} to~\ref{line:endpre})
in detail. The UCCGs (i.e., the connected components when all directed
edges are removed from $G$) are oriented one-by-one. For this, an
ordering $\tau$ is computed according to which the edges are oriented. This ordering $\tau$ is the
traversal sequence of an MCS. This algorithm is given in
lines~\ref{line:startmcs} to~\ref{line:endmcs}. It differs from
standard graph traversals in the way the
next vertex to visit is chosen (line~\ref{line:choosenext}). The
algorithm chooses one
of the vertices with most previously visited neighbors. This is
implemented efficiently by having a set of vertices with $0$
previously visited neighbors ($A[0]$), a set for vertices with $1$
previously visited neighbor ($A[1]$), and so on. When a vertex $v$ is
handled (lines~\ref{line:starthandle} to~\ref{line:endhandle}), its
neighbors are moved to a higher set (line~\ref{line:movehigher}).

This can be done in constant time per neighbor~\cite{tarjan1984simple},
meaning the MCS runs in linear-time. As it is performed on disjoint
subgraphs, we obtain overall linear-time $O(n+m)$ for computing a DAG
in $[G]$.

\begin{theorem}
  Algorithm~\ref{alg:cpdagextension} computes a DAG in the MEC
  represented by CPDAG $G$ in time $O(n+m)$.
\end{theorem}

\section{Missing Proofs and Algorithms} \label{appendix:missingproofs}
\subsection{Section 3}
We begin by giving the proof of Lemma~\ref{lemma:subgraph:over:s}:
\setcounter{lemma}{1}
\begin{lemma}
  Given a connected chordal graph $G = (V, E)$ and a sequence of
  visited vertices $\tau$ produced by an MCS with the
  current highest-label set $S$. Vertices $x,y \in S$ are connected in $G[S]$ iff they are connected in $G[V \setminus \tau]$.
\end{lemma}
\setcounter{lemma}{5}

\begin{proof}
  Trivially, if $x$ and $y$ are connected in $G[S]$ so they are in
  $G[V\setminus\tau]$. Hence, we only have to prove that if they are not
  connected in $G[S]$, then they are also disconnected in $G[V\setminus\tau]$.
  The special case that the highest label has label 0 is trivial as then we
  have $S = V \setminus \tau$. So let $\tau$ be non-empty and consider
  two arbitrary vertices $x,y\in S$ that are in different connected
  components of $G[S]$. For a contradiction assume that they are
  connected in $G[V\setminus\tau]$ via a shortest path
  $\pi = x-p_1-\dots-p_{\ell}-y$ with $\ell\geq 1$ and
  $p_1,\dots,p_{\ell}\in V\setminus\tau$. Note that $\pi$ is an unchorded
  path (i.\,e., one where only successive vertices are connected by an
  edge) due to being the shortest path between $x$ and $y$. Also denote
  the visited neighbors of vertex $v$ by $P(v) = \tau \cap \Ne(v)$.

  In $\pi$, there has to be a vertex $p_i \not\in S$, else $x$ and
  $y$ are connected in $G[S]$. Hence, there is (i) a vertex $z \in
  P(x)$, which is not in $P(p_i)$, and (ii) a vertex $z' \in P(y)$,
  which is not in $P(p_i)$, because $x$ and $y$ have higher label than
  $p_i$, meaning $|P(x)| > |P(p_i)|$ and $|P(y)| > |P(p_i)|$.

  We first consider the case that $z$ or $z'$ is a common neighbor of
  both $x$ and $y$ (this includes the case $z = z'$). Let this common
  neighbor be w.l.o.g.\ vertex $z$. Then, there is a cycle $x - p_1 -
  \dots - p_\ell - y - z - x$ in $G$. Moreover, $p_i$ is not a
  neighbor of $z$, vertices $x$ and $y$ are nonadjacent (else they
  would be connected in $G[S]$), and due to $\pi$ being the shortest
  path, there are no edges between $x - p_j$ for $j > 1$, $y - p_{j'}$
  for $j' < \ell$ and $p_j - p_{j'}$ for $|j-j'| > 1$. Hence, the only
  chords the cycle might have could be $p_j - z$ edges, with $j \neq
  i$. But there could only be $l-1$ such chords and for a cycle of
  length $l+3$, $l$ chords are needed to make it chordal. A
  contradiction.

  The remaining case is that $z$ and $z'$ are both not a common
  neighbor of $x$ and $y$. We show that there can be no AMO produced
  by the MCS, when it chooses $x$ or $y$ next from $S$, which would be
  a contradiction.
  W.l.o.g., we show the argument for $x$. Let $\alpha$ be a
  topological ordering inducing an AMO of $G$ having prefix $\tau +
  x$. Denote by $\alpha^{-1}(v)$ the position of $v$ in
  $\alpha$. It has to hold that $\alpha^{-1}(x) < \alpha^{-1}(p_1) <
  \alpha^{-1}(p_2) < \dots < \alpha^{-1}(p_\ell) < \alpha^{-1}(y)$ as the
  first vertex $p_j$ with $\alpha^{-1}(p_j) > \alpha^{-1}(p_{j+1})$,
  would induce a v-structure. As we have $\alpha^{-1}(z') < \alpha^{-1}(p_\ell) <
  \alpha^{-1}(y)$, vertices $z'$ and $p_\ell$ have to be adjacent to
  avoid a v-structure. Then, the same applies to $\alpha^{-1}(z') <
  \alpha^{-1}(p_{\ell-1}) < \alpha^{-1}(p_\ell)$, which implies an
  edge between $z'$ and $p_{\ell-1}$. This iteration can be continued only until
  vertex $p_i$, which is nonadjacent to $z'$ by assumption.
\end{proof}

For the sake of completeness, we also explicitly state the EnumMCS algorithm
for a CPDAG $G$. The enumeration task immediately reduces to finding the
AMOs of the chordal components of $G$.

\begin{algorithm}[h!]
  \caption{Linear-time delay enumeration algorithm of Markov
    equivalent DAGs based on EnumMCS.}
  \label{alg:enummcs}
  \DontPrintSemicolon
  \SetKwInOut{Input}{input}\SetKwInOut{Output}{output}
  \SetKwFor{Rep}{repeat}{}{end}
  \Input{A CPDAG $G = (V,E)$.}
  \Output{$[G]$.}
  \SetKwFunction{FMCS}{mcs}
  \SetKwProg{Fn}{function}{}{end}

  \BlankLine
  $U := (V, \{ (u,v) \mid (u,v) \text{ and } (v,u) \in E \})$, i.\,e.,
  the graph with only the undirected edges of $G$ \;
  Execute Algorithm~\ref{alg:mcs} on $U$ and each time add the
  directed edges of $G$ when outputting the DAG (in line~\ref{line:out} of
  Algorithm~\ref{alg:mcs}).
\end{algorithm}

\subsection{Section 4}
\label{subsection:appendixpdag}
We give the proof of Lemma~\ref{lemma:pdagfastreach} omitted in the
main paper:
\setcounter{lemma}{3}
\begin{lemma}
  Given a bucket $B$ and a sequence of visited vertices
  $\tau$ produced by the modified MCS using~$S^+$.
  Vertices $x,y \in S^+$ are connected in $B[V
  \setminus \tau]$ iff they are connected in $B[S^+]$.
\end{lemma}
\setcounter{lemma}{5}

\begin{proof}
  Recall that $S$ is the set of \emph{all} highest-label vertices including the
  ones with unvisited parents.
  As shown in the proof of Lemma~\ref{lemma:subgraph:over:s} the
  vertices $x$ and
  $y$ are connected in $S$ by the chordality of $B$.
  For a contradiction assume that $x$
  and~$y$ are connected in $S$ but not in $S^+$.
  Then there is a shortest path $x = p_1 - p_2 -\dots -p_{k-1} - p_k = y$ with some
  $p_i\in S\setminus S^+$. Consider the $p_i$ with
  smallest $i$ and assume the shortest path is chosen such that
  this $i$ is maximized (in case that there are multiple shortest
  paths). This path is completely 
  undirected as $x$ or $y$ would otherwise have an incoming edge by the same
  argument as in the proof of Lemma~\ref{lemma:pdagconn}.

  Observe that $p_i$ has an incoming edge from an unvisited vertex
  $z_1$. By the properties of the MCS
  we have $z_1\in S$ as otherwise $z_1$ would not be connected to a
  previous neighbor~$a$ of~$p_i$ implying the v-structure $a
  \rightarrow p_i \leftarrow c$ and violating the fact that the
  modified MCS returns an AMO. Therefore, $z_1$ needs to be connected to $p_{i-1}$. Furthermore, if $z_1$ has an incoming edge from an
  unvisited vertex~$z_2$ then this vertex has to be connected to
  $p_{i-1}$ as well.

  This process can be iterated further and we will
  consider the first~$z_j$ that has no unvisited parent. Such a vertex
  has to exist as otherwise there would be a directed cycle.
  The edge between~$z_j$ and~$p_{i-1}$ is undirected by the minimality
  of $p_i$.

  Moreover, $z_1$ needs to be adjacent to $p_{i+1}$ by the same
  argument. If this edge is undirected, $z_2$ needs to be adjacent to
  $p_{i+1}$. If it is directed $z_1 \rightarrow p_{i+1}$, then $z_1$
  needs to be adjacent to $p_{i+2}$. Generally, the highest index
  vertex on the path that $z_s$ is adjacent to is connected to it by
  an undirected edge. Meaning $z_{s+1}$ has to be connected to it as
  well. Thus, $z_j$ is connected by an undirected edge to some $p_t$
  on the path, for $t > i$. This yields the desired contradiction as
  either the path $x = p_1 - \dots - p_{i-1} - z_j - p_t - \dots - p_k
  = y$ is shorter than the previously considered one or the first
  vertex with an unvisited parent has a higher index.
\end{proof}

We explicitly state the generalization of EnumMCS to buckets, which
was only sketched in the main paper. The differences to Algorithm~\ref{alg:mcs} are \colorbox{orange!20}{highlighted}.
\begin{algorithm}
  \caption{Linear-time delay enumeration algorithm of consistent
    extensions of a bucket.}
  \label{alg:bucketmcs}
  \DontPrintSemicolon
  \SetKwInOut{Input}{input}\SetKwInOut{Output}{output}
  \SetKwFor{Rep}{repeat}{}{end}
  \Input{A bucket $B = (V,E)$.}
  \Output{All AMOs of $B$.}
  \SetKwFunction{FEnum}{enumerate}
  \SetKwProg{Fn}{function}{}{end}
  \SetKwRepeat{Do}{do}{while}

  \BlankLine
  \colorbox{orange!20}{$G := $ skeleton of $B$} \;
  $A :=$ array of $n$ initially empty sets \; 
  $\tau :=$ empty list \;
  $A[0] := V$ \;
  $\FEnum(G, A, \tau,$ \hspace*{-0.15cm} \colorbox{orange!20}{$B$} \hspace*{-0.15cm} $)$ \;
  \BlankLine
  
  \Fn{\FEnum{$G$, $A$, $\tau$, \hspace*{-0.15cm} \colorbox{orange!20}{$B$} \hspace*{-0.15cm}}}{
    \If{$|\tau| = n$}{Output AMO of $G$ according to ordering $\tau$ \label{line:bucketout}}
    $i := \text{highest index of non-empty set in } A$ \;
    \colorbox{orange!20}{$S^+ := \{ v \mid v \in A[i] \text{ and } \mathrm{Pa}_{B[V\setminus
      \tau]}(v) = \emptyset \}$} \;
    $v := $ \colorbox{orange!20}{$\text{any vertex from } S^+$} \;
    $x := v$ \;

    \Do{$\text{$R$ is non-empty}, x := \text{pop}(R)$}{
      $\text{delete } x \text{ from } A[i]$ \;
      $\text{append } x \text{ to } \tau$ \;
      \ForEach{$w \in (\Ne(x) \setminus \tau)$}{
        $j := \text{index of set in } A \text{ that } w \text{ is
          in}$ \;
        $\text{delete } w \text{ from } A[j]$ \;
        $\text{insert } w \text{ in } A[j+1]$ \;
      }
      \FEnum{$G$, $A$, $\tau$}\;
      \ForEach{$w \in (\Ne(x) \setminus \tau)$}{
        $j := \text{index of set in } A \text{ that } w \text{ is
          in}$ \;
        $\text{delete } w \text{ from } A[j]$ \;
        $\text{insert } w \text{ in } A[j-1]$ \;
      }
      $\text{insert } x \text{ in } A[i]$ \;
      $\text{pop } x \text{ from } \tau$ \;
      \If{$x = v$}{
        $R:= \{ a \mid a \text{ reachable from } v \text{ in }$
        \hspace*{-0.15cm}\colorbox{orange!20}{$G[S^+]$}\hspace*{-0.15cm} $\}$ \;
      }
      
    }
  }
\end{algorithm}

\begin{theorem}
  There is an algorithm that enumerates all
  AMOs of a given bucket with worst-case delay $O(n+m)$.
\end{theorem}

\begin{proof}
  This algorithm is given in Algorithm~\ref{alg:bucketmcs}.
  We first show that every DAG that is output by the algorithm is
  an AMO. This follows from the fact that the output is produced by an
  MCS. Notice that the
  restriction to vertices with no unvisited parents leads to no
  violation of this fact as the algorithm still \emph{only} chooses
  vertices with the highest-label. That such a vertex always exists
  was shown in~\cite{WienobstExtendability2021}.

  The next thing we prove is that every AMO of the given bucket is output. By
  Fact~\ref{fact:counting} every AMO can be
  represented by an MCS ordering. Which particular AMO is produced by an MCS depends on the
  specific choices of highest-label vertices.
  Our algorithm considers all possible
  choices at each step except (i) the ones where a vertex has unvisited
  parents and (ii)
  the ones where the vertex is unreachable from the first chosen vertex $v$
  in a recursion step.

  Clearly, (i) only excludes AMOs incompatible with the given directed
  edges (background knowledge). For (ii) observe that if a vertex is unreachable
  from $v$, it does not matter whether it is chosen before or after
  $v$ by item~\ref{lemma:pdagconn:b} of Lemma~\ref{lemma:pdagconn}.

  Finally, it remains to show that no AMO is output twice. Clearly,
  every sequence $\tau$ that we obtain is different. So assume for the
  sake of contradiction that the algorithm outputs two sequences~$\tau_1$ and
  $\tau_2$ that represent the
  same AMO. Let $x$ and $y$ be the vertices in $\tau_1$ and $\tau_2$ at
  the first differing position. Then $x$ and $y$ are
  connected in the induced subgraph over the unvisited vertices
  (by line~\ref{line:reach}) in the skeleton of $B$. A contradiction
  since this implies that
  the AMOs represented by
  $\tau_1$ and $\tau_2$ differ. 

  For the complexity analysis we partition the algorithm into three
  phases (i), (ii), and (iii) as in the proof of Theorem~\ref{theorem:wcdelay}.
  The cost of (ii) and (iii) are the
  same as in the analysis of Theorem~\ref{theorem:wcdelay}. For (i)
  just observe that reachability runs in $O(d)$ due to Lemma~\ref{lemma:pdagfastreach}.
%
%  The handling of the labels can be done similarly and taking care of the in-degree over unvisited vertices
%  is possible by having a counter per vertex.
\end{proof}

We also explicitly state the algorithm for enumerating PDAGs
(Algorithm\ref{alg:pdagextension}). The algorithm also works for
MPDAGs, in which case the \colorbox{orange!20}{highlighted} line~\ref{line:extrapdag} does not need to be executed.
The algorithm computes the
buckets and runs Algorithm~\ref{alg:bucketmcs} on them. 

\begin{algorithm}
  \caption{Linear-time delay enumeration algorithm of the consistent extensions of PDAG $G$ based on \textsc{mcs-enum} for buckets (Algorithm~\ref{alg:bucketmcs}).}
  \label{alg:pdagextension}
  \DontPrintSemicolon
  \SetKwInOut{Input}{input}\SetKwInOut{Output}{output}
  \SetKwFor{Rep}{repeat}{}{end}
  \Input{A PDAG $G = (V,E)$.}
  \Output{All consistent extensions of $G$.}
  \SetKwFunction{FMCS}{mcs}
  \SetKwProg{Fn}{function}{}{end}

  \BlankLine
  \If{$G$ has no consistent extensions}{
    \Return{$\emptyset$} \;
  }
  \colorbox{orange!20}{$G := $ the maximal orientation of $G$}
  \colorbox{orange!20}{(i.\,e., its completion under the Meek rules)} \label{line:extrapdag}
  \;
  $D := $ copy of $G$ \;
  $B := (V, \{ (u,v) \mid (u,v) \in E \allowbreak \text{ and } u,v \text{ are connected by
  undirected edges in } G \})$, i.\,e.,
  the graph containing only the buckets of $G$ \;
  
  Execute Algorithm~\ref{alg:bucketmcs} on $B$ and each time add the
  remaining edges of $G$ when outputting the DAG (in line~\ref{line:bucketout} of
  Algorithm~\ref{alg:bucketmcs}).
\end{algorithm}

\subsection{Section 5} 
\label{subsection:appendixleq3}

We complement the structural results in
Section~\ref{section:anotherapproach} in the main
paper, by explicitly stating the algorithm (called \textsc{shd3-enum}) described textually
in the proof of Theorem~\ref{theorem:dist3seq} and analyzing its
delay. The differences to \textsc{chickering-enum} (Algorithm~\ref{alg:chickering}) are \colorbox{orange!20}{highlighted}.
\begin{algorithm}[h!]
    \caption{Enumeration algorithm with SHD $\leq 3$ of Markov
    equivalent DAGs based on Chickerings'
    characterization~\cite{chickering1995transformational} (\textsc{shd3-enum}).}
  \label{alg:shd3}
  \DontPrintSemicolon
  \SetKwInOut{Input}{input}\SetKwInOut{Output}{output}
  \SetKwFor{Rep}{repeat}{}{end}
  \Input{A CPDAG $G = (V,E)$.}
  \Output{$[G]$.}
  \SetKwFunction{FDFS}{enumerate}
  \SetKwFunction{FExtend}{extend}
  \SetKwProg{Fn}{function}{}{end}

  \BlankLine
  $D := $ any DAG in $[G]$ \;
  $\mathrm{vis} := \{D\}$ \;
  $\FDFS(D, \mathrm{vis}, 0)$ \;
  \BlankLine
  
  \Fn{\FDFS{$D$, $\mathrm{vis}$, $i$}}{
    \If{\colorbox{orange!20}{$i \: \mathrm{mod} \: 2 = 0$}}{
      Output $D$ \;
    }

    \ForEach{DAG $D'$ obtained by reversing a covered edge in $D$}{
      \If{$D' \not\in \mathrm{vis}$}{
        $\mathrm{vis} := \mathrm{vis} \cup \{D'\}$ \;
        \FDFS{$D'$, $\mathrm{vis}$}
      }
    }

    \If{\colorbox{orange!20}{$i \: \mathrm{mod} \: 2 = 1$}}{
      \colorbox{orange!20}{Output $D$} \;
    }
  }
\end{algorithm}

The main paper uses this algorithm to show that it is possible to
enumerate an MEC with successive DAGs  having SHD at most three. Here,
we study the algorithmic properties of the algorithm.

\begin{theorem}\label{theorem:delayshd3}
  For a CPDAG $G$, \textsc{shd3-enum} enumerates the
  DAGs in $[G]$ with delay $O(m^2)$.
\end{theorem}

\begin{proof}
  The argument is similar to the proof of
  Theorem~\ref{theorem:chickeringdelay}. First, it is clear that all DAGs in
  the MEC are output exactly once.

  The analysis of the delay differs slightly in the sense that this
  algorithm has worst-case delay
  $O(m^2)$ (instead of $O(m^3)$) due to the fact that between two
  outputs only a constant number of recursive calls are handled. This
  follows from the fact that the algorithm outputs successive DAGs
  with SHD $\leq 3$ and, hence, between the output of these DAGs, there are only
  constantly many steps in the traversal of the MEC. The cost per DAG
  can be bounded by $O(m^2)$ as in Theorem~\ref{theorem:chickeringdelay}.
\end{proof}

Both Algorithm~\ref{alg:shd3} as well as Theorem~\ref{theorem:delayshd3} directly apply to PDAGs as well.

\section{MAGs -- Causal Graphs under Latent
  Confounding}\label{section:appendixmags}
MAGs without selection bias are mixed graphs $G = (V,E)$ with two
types of edges: directed
$x \rightarrow y$ and bidirected $x \leftrightarrow y$. The
semantics are different compared to DAGs in that edges encode
ancestral relations, i.\,e., a directed edge means that $x$ is an
ancestor (a cause, but not necessarily a direct cause) of $y$ and a
bidirected edge means that
neither $x$ is a cause of $y$ nor is $y$ a cause of $x$ (as
simplification, think that there is a latent confounder between $x$
and $y$). An excellent introduction to MAGs is given
in~\cite{zhang2008completeness}.

\citet{zhang2005transformational} gave the following transformational
characterization of Markov
equivalent MAGs without selection bias (also called \emph{DMAGs}):
\begin{theorem}[\citet{zhang2005transformational}] \label{theorem:magtrans}
  Two DMAGs $G$ and $G'$ are Markov equivalent iff there exists a
  sequence of single edge mark changes in $G$ such that
  \begin{enumerate}
    \item after each mark change, the resulting graph is also a DMAG
      and is Markov equivalent to $G$,
    \item after all the mark changes, the resulting graph is $G'$.
  \end{enumerate}
\end{theorem}

It follows that the graph over Markov equivalent DMAGs with an edge
between DMAGs with SHD 1 (we define the SHD similar to the DAG case, that
is, the number of differing edges; one could also define it to be the
number of differing edge marks, and the statement still holds)
is connected and hence, there is, by the same argument as in
Theorem~\ref{theorem:dist3seq}, a sequence of DMAGs containing each in the
MEC exactly once with
successive DMAGs having SHD $\leq 3$, which proves
Corollary~\ref{corollary:mags}.

We consider it out of the scope of this work to investigate the
algorithmic aspects of enumerating Markov equivalent MAGs in detail, but deem
this an fascinating topic for further research. In particular, it
would be highly interesting to develop enumeration algorithms based on
the other strategies discussed in this paper. 

Finally, we note that the transformational characterization in Theorem~\ref{theorem:magtrans} does not hold for MAGs \emph{with} selection bias and it follows
from the example given in~\cite{zhang2005transformational} that it is indeed not
possible to find an enumeration sequence with SHD $\leq 3$ in this case.

\section{Further Experimental Results} \label{appendix:experiments}
This section shows the full experimental results of the four algorithms
(\textsc{meek-enum}, \textsc{chickering-enum}, \textsc{mcs-enum},
\textsc{shd3-enum}) with results for CPDAGs and PDAGs, as well as for
slightly denser graphs. Fig.~\ref{fig:plots} depicts the results.
The instances are generated as follows:
\begin{itemize}
\item For the undirected chordal graphs, a
  random tree was generated, to which randomly drawn edges are
  inserted, which do not violate chordality. This is done until the
  graph has $k\cdot n$ edges. In the main paper, we considered
  $k=3$, here we choose slightly denser graphs with $k = \log_2 n$.
\item The CPDAGs are generated by first creating a DAG $D$ randomly, whose
  CPDAG representation is computed afterward. The DAG is, on the one
  hand, sampled analogously to the undirected case, by adding
  directed edges, which do not violate acyclicity until $k \cdot n$
  edges are reached. On the other hand, we generate DAGs based on scale-free
  graphs (using
  the Barabási-Albert model~\citep{BarabasiAlbert2002}), which are
  oriented by imposing a random topological ordering.
\item The PDAGs are generated by first sampling CPDAGs (in the way
  described above) and then orienting further edges (imitating the
  addition of background knowledge). For this, between three and
  seven undirected edges (this number is chosen uniformly at random)
  are oriented in the CPDAG, with the Meek rules being applied after
  each orientation.
\end{itemize}
For CPDAGs and PDAGs, as for the chordal graphs, we choose the
parameters $k = 3$ and $k = \log_2 n$.
The four algorithms were run on each of the instances for a maximum of
one minute and the delay between outputs was averaged. In case of the
smaller instances, the faster algorithms were often able to enumerate
the whole MEC, for larger instances this is not realistically possible
due to their vast size.

For the experiments, we used a single core of the AMD Ryzen
Threadripper 3970X 32-core processor on a 256GB RAM
machine.\footnote{The experiments can also be run on a desktop
  computer without any problems.}
We report the times in milliseconds and excluded averages over 32ms
from the plots to keep the results visually comparable. 

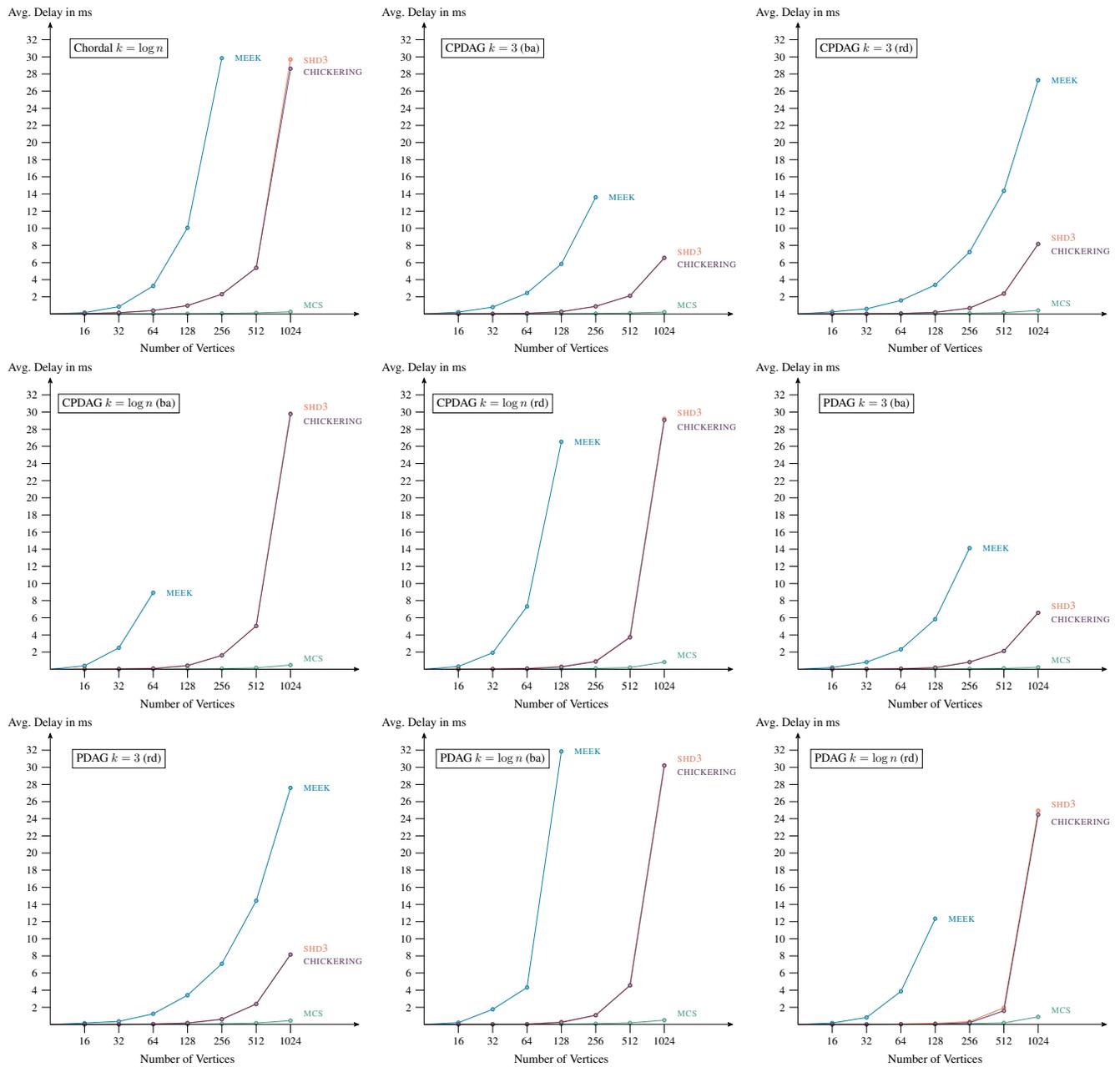
\begin{figure*}
  \begin{minipage}[b]{0.33\textwidth}
    \resizebox{\textwidth}{!}{
      \begin{tikzpicture}

        \begin{scope}[yscale=0.25]
          \plot[0.2cm]{
            0.00639,
            0.00981,
            0.01581,
            0.02762,
            0.05804,
            0.11896,
            0.24992
          }{ba.pine}{\textsc{mcs}}
        
          \plot{
            0.15888,
            0.84746,
            3.26362,
            10.05574,
            29.84997
        %    83.25906,
        %    317.94848
          }{ba.blue}{\textsc{meek}}
        
          \plot{
            0.04143,
            0.14510,
            0.40074,
            0.97556,
            2.30268,
            5.38441,
            29.70504
          }{ba.orange}{\textsc{shd3}}
        
          \plot[-0.1cm]{
            0.04120,
            0.14653,
            0.39966,
            0.97576,
            2.30175,
            5.37991,
            28.62709
          }{ba.violet}{\textsc{chickering}}
        
          \draw[semithick, ->, >={[round]Stealth}] (0,0) -- (0, 34) node[above] {Avg.\ Delay in ms};
          \foreach \y in {2,4,...,32} {
            \draw (0,\y) -- (-.25,\y) node[left] {\small\y};
          }

          \node[draw] at (2,31) {Chordal $k = \log n$};
          
        \end{scope}
        
        \draw[semithick, ->, >={[round]Stealth}] (0,0) -- (9,0);
        \foreach [count=\x] \label in {16, 32, 64, 128, 256, 512, 1024}{
          \draw (\x,0) -- (\x, -0.25) node[below] {\small\label};
        }
        \node at (4,-1) {Number of Vertices};
        \end{tikzpicture}
      }
  \end{minipage}
  \begin{minipage}[b]{0.33\textwidth}
    \resizebox{\textwidth}{!}{
      \begin{tikzpicture}

        \begin{scope}[yscale=0.25]
          \plot[0.2cm]{
            0.00539,
            0.00741,
            0.01426,
            0.02740,
            0.05258,
            0.10386,
            0.21898
          }{ba.pine}{\textsc{mcs}}
        
          \plot{
            0.21428,
            0.79459,
            2.44158,
            5.83786,
            13.61720
        %    33.84126
        %    88.50870
          }{ba.blue}{\textsc{meek}}
        
          \plot[0.2cm]{
            0.01228,
            0.02216,
            0.07509,
            0.26078,
            0.88550,
            2.12513,
            6.55741
          }{ba.orange}{\textsc{shd3}}
        
          \plot[-0.2cm]{
            0.01195,
            0.02116,
            0.07214,
            0.26183,
            0.88428,
            2.12575,
            6.53999
          }{ba.violet}{\textsc{chickering}}
        
          \draw[semithick, ->, >={[round]Stealth}] (0,0) -- (0, 34) node[above] {Avg.\ Delay in ms};
          \foreach \y in {2,4,...,32} {
            \draw (0,\y) -- (-.25,\y) node[left] {\small\y};
          }

          \node[draw] at (2,31) {CPDAG $k = 3$ (ba)};
        \end{scope}
        
        \draw[semithick, ->, >={[round]Stealth}] (0,0) -- (9,0);
        \foreach [count=\x] \label in {16, 32, 64, 128, 256, 512, 1024}{
          \draw (\x,0) -- (\x, -0.25) node[below] {\small\label};
        }
        \node at (4,-1) {Number of Vertices};
      \end{tikzpicture}
    }
  \end{minipage}
  \begin{minipage}[b]{0.33\textwidth}
    \resizebox{\textwidth}{!}{
      \begin{tikzpicture}

        \begin{scope}[yscale=0.25]
          \plot[0.2cm]{
            0.00448,
            0.00749,
            0.01459,
            0.02885,
            0.08048,
            0.15156,
            0.41258
          }{ba.pine}{\textsc{mcs}}
        
          \plot{
            0.24815,
            0.59846,
            1.57548,
            3.39190,
            7.22890,
            14.36171,
            27.27915
          }{ba.blue}{\textsc{meek}}
        
          \plot[0.2cm]{
            0.01039,
            0.01971,
            0.05799,
            0.19221,
            0.68171,
            2.37504,
            8.17005
          }{ba.orange}{\textsc{shd3}}
        
          \plot[-0.2cm]{
            0.00953,
            0.01802,
            0.04857,
            0.18847,
            0.68081,
            2.37082,
            8.16770
          }{ba.violet}{\textsc{chickering}}
        
          \draw[semithick, ->, >={[round]Stealth}] (0,0) -- (0, 34) node[above] {Avg.\ Delay in ms};
          \foreach \y in {2,4,...,32} {
            \draw (0,\y) -- (-.25,\y) node[left] {\small\y};
          }

          \node[draw] at (2,31) {CPDAG $k = 3$ (rd)};
          
        \end{scope}
        
        \draw[semithick, ->, >={[round]Stealth}] (0,0) -- (9,0);
        \foreach [count=\x] \label in {16, 32, 64, 128, 256, 512, 1024}{
          \draw (\x,0) -- (\x, -0.25) node[below] {\small\label};
        }
        \node at (4,-1) {Number of Vertices};
      \end{tikzpicture}
    }
  \end{minipage}

  \begin{minipage}[b]{0.33\textwidth}
    \resizebox{\textwidth}{!}{
      \begin{tikzpicture}

        \begin{scope}[yscale=0.25]
          \plot[0.2cm]{
            0.00456,
            0.00811,
            0.01647,
            0.03214,
            0.06770,
            0.16386,
            0.48648
          }{ba.pine}{\textsc{mcs}}
        
          \plot{
            0.40117,
            2.50031,
            8.92329
        %    38.13017
        %    126.34994,
        %    435.64693,
        %    1659.56742
          }{ba.blue}{\textsc{meek}}
        
          \plot[0.2cm]{
            0.01303,
            0.03238,
            0.09168,
            0.42454,
            1.61729,
            5.04493,
            29.82250
          }{ba.orange}{\textsc{shd3}}
        
          \plot[-0.2cm]{
            0.01253,
            0.02924,
            0.07426,
            0.41996,
            1.61071,
            5.04514,
            29.76576
          }{ba.violet}{\textsc{chickering}}
        
          \draw[semithick, ->, >={[round]Stealth}] (0,0) -- (0, 34) node[above] {Avg.\ Delay in ms};
          \foreach \y in {2,4,...,32} {
            \draw (0,\y) -- (-.25,\y) node[left] {\small\y};
          }

          \node[draw] at (2,31) {CPDAG $k = \log n$ (ba)};
        
        \end{scope}
        
        \draw[semithick, ->, >={[round]Stealth}] (0,0) -- (9,0);
        \foreach [count=\x] \label in {16, 32, 64, 128, 256, 512, 1024}{
          \draw (\x,0) -- (\x, -0.25) node[below] {\small\label};
        }
        \node at (4,-1) {Number of Vertices};
      \end{tikzpicture}
    }
  \end{minipage}
  \begin{minipage}[b]{0.33\textwidth}
    \resizebox{\textwidth}{!}{
      \begin{tikzpicture}
        
        \begin{scope}[yscale=0.25]
          \plot[0.2cm]{
            0.00442,
            0.00899,
            0.01664,
            0.03602,
            0.10326,
            0.20601,
            0.83299
          }{ba.pine}{\textsc{mcs}}
        
          \plot{
            0.32262,
            1.92468,
            7.31545,
            26.53795
        %    85.45591,
        %    273.81451,
        %    640.09755
          }{ba.blue}{\textsc{meek}}
        
          \plot[0.2cm]{
            0.01199,
            0.02740,
            0.09595,
            0.31362,
            0.94185,
            3.78916,
            29.22715
          }{ba.orange}{\textsc{shd3}}
        
          \plot[-0.2cm]{
            0.01075,
            0.02427,
            0.06984,
            0.26351,
            0.88946,
            3.71220,
            29.06266
          }{ba.violet}{\textsc{chickering}}
        
          \draw[semithick, ->, >={[round]Stealth}] (0,0) -- (0, 34) node[above] {Avg.\ Delay in ms};
          \foreach \y in {2,4,...,32} {
            \draw (0,\y) -- (-.25,\y) node[left] {\small\y};
          }

          \node[draw] at (2,31) {CPDAG $k = \log n$ (rd)};
          
        \end{scope}
        
        \draw[semithick, ->, >={[round]Stealth}] (0,0) -- (9,0);
        \foreach [count=\x] \label in {16, 32, 64, 128, 256, 512, 1024}{
          \draw (\x,0) -- (\x, -0.25) node[below] {\small\label};
        }
        \node at (4,-1) {Number of Vertices};
      \end{tikzpicture}
    }
  \end{minipage}
  \begin{minipage}[b]{0.33\textwidth}
    \resizebox{\textwidth}{!}{
      \begin{tikzpicture}

        \begin{scope}[yscale=0.25]
          \plot[0.2cm]{
            0.00466,
            0.00788,
            0.01524,
            0.02884,
            0.04541,
            0.10797,
            0.22775
          }{ba.pine}{\textsc{mcs}}
        
          \plot{
            0.20064,
            0.82930,
            2.31525,
            5.83795,
            14.13217
        %    33.20019,
        %    88.37550
          }{ba.blue}{\textsc{meek}}
        
          \plot[0.2cm]{
            0.00942,
            0.01839,
            0.05837,
            0.18607,
            0.83354,
            2.12790,
            6.59496
          }{ba.orange}{\textsc{shd3}}
        
          \plot[-0.2cm]{
            0.00909,
            0.01612,
            0.04999,
            0.18654,
            0.83323,
            2.12404,
            6.58511
          }{ba.violet}{\textsc{chickering}}
        
          \draw[semithick, ->, >={[round]Stealth}] (0,0) -- (0, 34) node[above] {Avg.\ Delay in ms};
          \foreach \y in {2,4,...,32} {
            \draw (0,\y) -- (-.25,\y) node[left] {\small\y};
          }

          \node[draw] at (2,31) {PDAG $k = 3$ (ba)};
        
        \end{scope}
        
        \draw[semithick, ->, >={[round]Stealth}] (0,0) -- (9,0);
        \foreach [count=\x] \label in {16, 32, 64, 128, 256, 512, 1024}{
          \draw (\x,0) -- (\x, -0.25) node[below] {\small\label};
        }
        \node at (4,-1) {Number of Vertices};
      \end{tikzpicture}
    }
  \end{minipage}

  \begin{minipage}[b]{0.33\textwidth}
    \resizebox{\textwidth}{!}{
      \begin{tikzpicture}

        \begin{scope}[yscale=0.25]
          \plot[0.2cm]{
            0.00396,
            0.00660,
            0.01592,
            0.02432,
            0.07861,
            0.15380,
            0.44911
          }{ba.pine}{\textsc{mcs}}
        
          \plot{
            0.15423,
            0.35640,
            1.23918,
            3.40947,
            7.08301,
            14.42480,
            27.60843
          }{ba.blue}{\textsc{meek}}
        
          \plot[0.2cm]{
            0.00642,
            0.01043,
            0.04778,
            0.16948,
            0.60089,
            2.38132,
            8.15965
          }{ba.orange}{\textsc{shd3}}
        
          \plot[-0.2cm]{
            0.00570,
            0.00867,
            0.03957,
            0.15688,
            0.60151,
            2.38180,
            8.15417
          }{ba.violet}{\textsc{chickering}}
        
          \draw[semithick, ->, >={[round]Stealth}] (0,0) -- (0, 34) node[above] {Avg.\ Delay in ms};
          \foreach \y in {2,4,...,32} {
            \draw (0,\y) -- (-.25,\y) node[left] {\small\y};
          }

          \node[draw] at (2,31) {PDAG $k = 3$ (rd)};
        
        \end{scope}
        
        \draw[semithick, ->, >={[round]Stealth}] (0,0) -- (9,0);
        \foreach [count=\x] \label in {16, 32, 64, 128, 256, 512, 1024}{
          \draw (\x,0) -- (\x, -0.25) node[below] {\small\label};
        }
        \node at (4,-1) {Number of Vertices};
      \end{tikzpicture}
    }
  \end{minipage}
  \begin{minipage}[b]{0.33\textwidth}
    \resizebox{\textwidth}{!}{
      \begin{tikzpicture}

        \begin{scope}[yscale=0.25]
          \plot[0.2cm]{
            0.00394,
            0.00855,
            0.01542,
            0.03572,
            0.08319,
            0.17859,
            0.49506
          }{ba.pine}{\textsc{mcs}}
        
          \plot{
            0.20230,
            1.76054,
            4.31144,
            31.85532
        %    123.05392,
        %    446.48728,
        %    1635.13851
          }{ba.blue}{\textsc{meek}}
        
          \plot[0.2cm]{
            0.00583,
            0.01706,
            0.04469,
            0.29389,
            1.08789,
            4.56173,
            30.23766
          }{ba.orange}{\textsc{shd3}}
        
          \plot[-0.2cm]{
            0.00517,
            0.01524,
            0.01541,
            0.22904,
            1.07006,
            4.55606,
            30.19966
          }{ba.violet}{\textsc{chickering}}
        
          \draw[semithick, ->, >={[round]Stealth}] (0,0) -- (0, 34) node[above] {Avg.\ Delay in ms};
          \foreach \y in {2,4,...,32} {
            \draw (0,\y) -- (-.25,\y) node[left] {\small\y};
          }

          \node[draw] at (2,31) {PDAG $k = \log n$ (ba)};
          
        \end{scope}
        
        \draw[semithick, ->, >={[round]Stealth}] (0,0) -- (9,0);
        \foreach [count=\x] \label in {16, 32, 64, 128, 256, 512, 1024}{
          \draw (\x,0) -- (\x, -0.25) node[below] {\small\label};
        }
        \node at (4,-1) {Number of Vertices};
      \end{tikzpicture}
    }
  \end{minipage}
  \begin{minipage}[b]{0.33\textwidth}
    \resizebox{\textwidth}{!}{
      \begin{tikzpicture}

        \begin{scope}[yscale=0.25]
          \plot[0.2cm]{
            0.00355,
            0.00627,
            0.01590,
            0.03054,
            0.09422,
            0.17599,
            0.88358
          }{ba.pine}{\textsc{mcs}}
        
          \plot{
            0.16140,
            0.81208,
            3.85474,
            12.34416
        %    46.02540,
        %    168.17881,
        %    621.68185
          }{ba.blue}{\textsc{meek}}
        
          \plot[0.2cm]{
            0.00388,
            0.00661,
            0.05183,
            0.12269,
            0.32213,
            1.93143,
            24.93001
          }{ba.orange}{\textsc{shd3}}
        
          \plot[-0.2cm]{
            0.00386,
            0.00628,
            0.02186,
            0.03918,
            0.20031,
            1.58934,
            24.47453
          }{ba.violet}{\textsc{chickering}}
        
          \draw[semithick, ->, >={[round]Stealth}] (0,0) -- (0, 34) node[above] {Avg.\ Delay in ms};
          \foreach \y in {2,4,...,32} {
            \draw (0,\y) -- (-.25,\y) node[left] {\small\y};
          }

          \node[draw] at (2,31) {PDAG $k = \log n$ (rd)};
          
        \end{scope}
        
        \draw[semithick, ->, >={[round]Stealth}] (0,0) -- (9,0);
        \foreach [count=\x] \label in {16, 32, 64, 128, 256, 512, 1024}{
          \draw (\x,0) -- (\x, -0.25) node[below] {\small\label};
        }
        \node at (4,-1) {Number of Vertices};
      \end{tikzpicture}
    }
  \end{minipage}
  \caption{Comparison of the average delay between two output DAGs for the algorithms
    \textcolor{ba.blue}{\textsc{meek-enum}},
    \textcolor{ba.violet}{\textsc{chickering-enum}},
    \textcolor{ba.pine}{\textsc{mcs-enum}}, and
    \textcolor{ba.orange}{\textsc{shd3-enum}}. The graphs are
    generated as described in the text with density parameter $k$ and
    the generation methods (ba), for scale-free graphs, and (rd), for
    uniformly random edge insertions.}
  \label{fig:plots}
\end{figure*}

All results in Fig.~\ref{fig:plots} show the
same pattern, namely that \textsc{mcs-enum} is superior to the other
algorithms and both \textsc{chickering-enum} and \textsc{shd3-enum} outperform
\textsc{meek-enum}. The superiority of \textsc{mcs-enum} can be easily
explained by the fact that its competitors have cost at least in the size of
the graph $O(n+m)$ after every (re)-orientation of an edge. While this
is the \emph{total cost} between two outputs for
\textsc{mcs-enum}. The algorithms \textsc{chickering-enum} and
\textsc{shd3-enum} have a very similar average delay, only the order
of the output differs (in the way that \textsc{shd3-enum}
guarantees smoothly changing DAGs). \textsc{meek-enum} has by far the
highest cost, due to the large effort arising from the repeated
completion under the Meek rules. 

Finally, we take a look at the distribution of delays, which we also
recorded during the measurements.
In Table~\ref{table:delays}, the percentage of delays $d$ with
$d \leq k \cdot \mathrm{mean}$ is given for each algorithm and different $k$
in every considered scenario.

\begin{table*}
  \centering
  \begin{tabular}{llllllll} \toprule
    Algorithm                & Scenario & $k = 1$ & $k = 2$ & $k = 3$ & $k = 5$ & $k = 7$ & $k = 10$ \\
    \cmidrule(l){1-2} \cmidrule(l){3-8}
    \textsc{meek-enum}       & UCCG     & 69.56   & 91.26   & 96.69   & 98.86   & 99.25   & 99.68 \\
                             & CPDAG    & 63.37   & 90.52   & 97.27   & 99.74   & 99.95   & 99.98 \\
                             & PDAG     & 63.61   & 90.79   & 97.35   & 99.75   & 99.95   & 99.98 \\
                             & All      & 69.35   & 91.24   & 96.71   & 98.89   & 99.27   & 99.69 \\
    \cmidrule(l) {1-2} \cmidrule(l) {3-8}
    \textsc{chickering-enum} & UCCG     & 89.83   & 99.73   & 99.86   & 99.87   & 99.88   & 99.89 \\
                             & CPDAG    & 86.04   & 99.77   & 99.79   & 99.80   & 99.81   & 99.83 \\
                             & PDAG     & 81.58   & 99.73   & 99.79   & 99.81   & 99.83   & 99.85 \\
                             & All      & 89.02   & 99.73   & 99.85   & 99.86   & 99.87   & 99.88 \\ \cmidrule(l) {1-2} \cmidrule(l) {3-8}
    \textsc{mcs-enum}        & UCCG     & 91.98   & 99.68   & 99.77   & 99.80   & 99.81   & 99.81 \\
                             & CPDAG    & 76.87   & 98.93   & 99.60   & 99.66   & 99.66   & 99.66 \\
                             & PDAG     & 74.29   & 98.42   & 98.49   & 99.63   & 99.63   & 99.63 \\
                             & All      & 89.96   & 99.56   & 99.75   & 99.78   & 99.79   & 99.79 \\ \cmidrule(l) {1-2} \cmidrule(l) {3-8}
    \textsc{shd3-enum}       & UCCG     & 77.39   & 99.37   & 99.86   & 99.87   & 99.88   & 99.89 \\
                             & CPDAG    & 67.81   & 98.90   & 99.79   & 99.80   & 99.81   & 99.83 \\
                             & PDAG     & 63.35   & 98.43   & 99.79   & 99.81   & 99.83   & 99.85 \\
                             & All      & 75.73   & 99.27   & 99.85   & 99.86   & 99.87   & 99.88 \\ \bottomrule
  \end{tabular}
  \caption{The proportion of delays (in percent) that are less or equal to multiples
    of the mean delay $\bar{d}$, i.\,e., the percentage of delays $d$ for which
    $d \leq k \cdot \bar{d}$ holds.
    We evaluated the delays for each algorithm in a specific scenario and for all
    scenarios combined.}
  \label{table:delays}
\end{table*}

It becomes evident that more than 90 percent of the measured delays are not
larger than two times the mean.
For the algorithms \textsc{chickering-enum}, \textsc{mcs-enum}, and
\textsc{shd3-enum}, the proportion is even bigger, often between 98 and 99 percent.
\textsc{meek-enum} has the biggest proportion of delays being larger than
twice the mean, which can be explained by the fact that some sub-steps
need significantly more time to apply Meek's rules than others.
Outliers being larger than seven times the mean are nearly as often as outliers
being larger than ten times the mean, showing that there are few (in fact, less
than 0.5 percent) delays that are way above the mean. We conjecture
that these infrequent outliers might be due to garbage
collection during the measurements (which becomes a noticeable factor
in millisecond measurements). 

\end{document}